\documentclass{article}

\usepackage[main, final]{neurips_2026}

\usepackage{amsmath}
\usepackage{amssymb}
\usepackage{bm}
\usepackage{mathtools}
\usepackage{amsthm}
\usepackage[whole]{bxcjkjatype} 

\usepackage{algorithm}
\usepackage{algpseudocode}
\usepackage{xcolor}
\definecolor{algvariant}{HTML}{0B6E4E} 
\usepackage{fvextra}
\definecolor{qualoutframe}{HTML}{8899BB}

\usepackage[inline]{enumitem}
\usepackage{listings}
\lstdefinestyle{chattemplate}{
  basicstyle=\ttfamily\tiny,
  columns=fullflexible,
  breaklines=true,
  frame=single,
  keepspaces=true,
  showstringspaces=false
}
\usepackage{hyperref}
\hypersetup{
  colorlinks   = true, 
  urlcolor     = [rgb]{0,0,1}, 
  linkcolor    = [rgb]{0,0,0.5}, 
  citecolor   = [rgb]{0,0,0.5} 
}

\makeatletter
\let\MYcaption\@makecaption
\makeatother

\usepackage{subcaption}
\makeatletter
\let\@makecaption\MYcaption
\makeatother

\usepackage[capitalize,noabbrev]{cleveref}
\usepackage{relsize}

\usepackage{booktabs}
\usepackage{array}
\usepackage{multirow}
\usepackage{tikz}
\usetikzlibrary{positioning,arrows.meta,fit,backgrounds,calc}
\theoremstyle{plain}

\newtheorem{proposition}{Proposition}
\newtheorem{lemma}{Lemma}
\newtheorem{corollary}{Corollary}
\newtheorem{definition}{Definition}

\theoremstyle{remark}

\crefname{theorem}{Theorem}{Theorems}
\Crefname{theorem}{Theorem}{Theorems}

\crefname{proposition}{Proposition}{Propositions}
\Crefname{proposition}{Proposition}{Propositions}

\crefname{lemma}{Lemma}{Lemmas}
\Crefname{lemma}{Lemma}{Lemmas}

\crefname{corollary}{Corollary}{Corollaries}
\Crefname{corollary}{Corollary}{Corollaries}

\crefname{definition}{Definition}{Definitions}
\Crefname{definition}{Definition}{Definitions}

\crefname{assumption}{Assumption}{Assumptions}
\Crefname{assumption}{Assumption}{Assumptions}

\crefname{remark}{Remark}{Remarks}
\Crefname{remark}{Remark}{Remarks}

\usepackage[textsize=tiny]{todonotes}

\usepackage{fontawesome5} 

\newcommand{\argmax}{\mathop{\rm arg~max}\limits}

\usepackage{authblk}
\author{
{\bf Akiyoshi Tomihari$^{1}$\quad Issei Sato$^{1}$}
\\
Department of Computer Science, The University of Tokyo, Japan
\vspace{1em}
\\
\texttt{tomihari@g.ecc.u-tokyo.ac.jp, sato@g.ecc.u-tokyo.ac.jp}
}
\date{}

\begin{document}

\title{Power Distribution Bridges Sampling, Self-Reward RL, and Self-Distillation}
\maketitle

\begin{abstract}
Recent analyses question whether reinforcement learning (RL) is responsible for strong reasoning in large language models (LLMs).
At the same time, distillation and inference-time sampling, including power sampling, have emerged as effective ways to improve LLM performance.
However, the relationship among RL, distillation, and sampling remains unclear.
In this study, we focus on the power distribution, the target distribution of power sampling, and show that the power distribution bridges sampling, self-reward KL-regularized RL, and self-distillation.
From the sampling perspective, we show that inexpensive local approximations cannot reproduce sequence-level power without information about possible suffixes.
From the RL perspective, the power distribution is the closed-form optimizer of KL-regularized RL when the model's sequence-level log-probabilities are used as the reward.
This identification leads to \emph{power self-distillation}, an offline distillation surrogate that shares the same target distribution and amortizes the cost of power sampling into supervised training on teacher samples.
We further show that power self-distillation can achieve self-reward sharpening, that downstream true-reward gains are governed by the covariance between true reward and self-reward under the power distribution, and that this ideal power-teacher gain transfers to a distilled student when the student is close to the power teacher.
Experiments on reasoning tasks support our analysis: power sampling raises self-reward, true-reward gains depend on alignment with self-reward, and power self-distillation can match or exceed the performance of power sampling at much lower inference cost. The code is available at \href{https://github.com/tom4649/power-self-distillation}{\faGithub\ tom4649/power-self-distillation}.
\end{abstract}

\section{Introduction}
\label{sec:introduction}
The strong reasoning ability exhibited by large language models (LLMs) has often been attributed to reinforcement learning (RL).
However, empirical analyses with extensive sampling question whether RL alone fully explains post-RL reasoning gains: as the number of sampled generations grows, post-RL models often fail to outperform their corresponding pre-RL counterparts~\citep{yue2025does}.
At the same time, distillation has become a standard way to transfer the capabilities of expensive or stronger models to smaller models~\citep{hinton2015distilling,guo2025deepseek,busbridge2025distillation}, and inference-time compute allocated to sampling or search has improved LLM performance~\citep{snell2024scaling,welleck2024from}.

However, the relationship between sampling, RL, and self-distillation remains unclear.
In particular, \citet{karan2026reasoning} show that a base model, without additional training or external reward, can match or exceed post-RL models using \emph{power sampling}.
This raises the question of whether the success of power sampling reflects a mechanism distinct from RL and distillation, or whether these methods can be connected through a common structure.
Clarifying such a connection is important because it can reveal whether gains that appear to come from different procedures in fact arise from a common mechanism, and whether an expensive inference-time procedure can be converted into an offline training objective.

In this study, we show that sampling, RL, and self-distillation are naturally connected through the \emph{power distribution}.
As illustrated in~\Cref{fig:contribution}, this distribution is the target of power sampling, the closed-form optimum of a self-reward RL objective, and the teacher distribution amortized by self-distillation.
From the sampling perspective, a natural question is whether the effect of power sampling can be reproduced by an inexpensive token-level approximation.
We show that this is structurally difficult: per-token approximations cannot match the power distribution without sequence-level information.
From the RL perspective, the power distribution is the closed-form optimum of KL-regularized RL~\citep{ouyang2022instructgpt} when the reward is the model's sequence-level log-probabilities, i.e., the self-reward in the sense of \citet{huang2025selfimprovement}.
Finally, by rewriting this RL objective, we derive \emph{power self-distillation} as an offline distillation surrogate that shares the same target distribution and amortizes the cost of power sampling into offline training.
We further show that power self-distillation admits a self-reward sharpening guarantee, that the true-reward gain of the power teacher is determined by a reward covariance under the power distribution, and that this gain transfers to a distilled student when the student remains close to the power teacher.

Our contributions are summarized as follows. \Cref{fig:contribution} illustrates the connection we study, and \Cref{tab:comparison} compares these axes with prior work.
\begin{itemize}[leftmargin=1.0em,topsep=0pt,itemsep=1pt,parsep=0pt,partopsep=0pt]
  \item We show that approximating power sampling at inference time is structurally hard: per-token approximations cannot match the power distribution without sequence-level information (\Cref{prop:pow-vs-temp-renyi,prop:opt-proposal-power}).
  \item We show that the power distribution is the closed-form optimum of KL-regularized RL with the model's sequence-level log-probabilities as the reward (\Cref{cor:self-reward-power}), and derive \emph{power self-distillation} by rewriting this RL objective, thereby amortizing expensive power sampling into offline training (\Cref{alg:offline-power-distill}).
  \item We provide a self-reward sharpening bound for power self-distillation (\Cref{prop:teacher-sft}), characterize the true-reward gain of the power teacher through a covariance condition under the power distribution (\Cref{prop:Rprime-cov}), and show that this gain transfers to a distilled student up to a forward-KL approximation term (\Cref{prop:gain-transfer}).
\end{itemize}

\begin{figure}[t]
  \centering
  \footnotesize
  \resizebox{\linewidth}{!}{
\begin{tikzpicture}[
  font=\footnotesize,
  >=Stealth,
  node distance=8mm,
  method/.style={
    rectangle,
    rounded corners=3pt,
    draw=black!75,
    line width=0.65pt,
    align=center,
    inner sep=4pt,
    minimum height=23mm,
    text width=28mm,
  },
  centerbox/.style={
    rectangle,
    rounded corners=4pt,
    draw=black!80,
    line width=0.85pt,
    align=center,
    inner sep=5pt,
    minimum height=16mm,
    text width=28mm,
    fill=gray!6,
  },
  edge/.style={-, line width=0.75pt, shorten <=2pt, shorten >=2pt},
  reflabel/.style={
    font=\scriptsize,
    align=center,
    fill=white,
    inner sep=1.5pt,
    text=black!85,
  },
  deriveedge/.style={->, dashed, line width=0.7pt, draw=black!65, shorten <=2pt, shorten >=2pt},
]

\colorlet{samplingcolor}{orange!11}
\colorlet{rlcolor}{blue!9}
\colorlet{distillcolor}{teal!10}

\node[centerbox] (power) {%
  \textbf{Power distribution}\\[2pt]
  $\displaystyle \pi_\alpha(y\mid x)\propto \pi(y\mid x)^\alpha$
};

\node[method, fill=samplingcolor, left=24mm of power] (sampling) {%
  \textbf{Sampling}\\[2pt]
  Inference-time \\[-1pt]
  approximation of $\pi_\alpha$\\[2pt]
  \begin{tikzpicture}[x=3mm,y=8mm,baseline=-0.6ex]
    \draw[->, black!55] (-0.3,0) -- (7.3,0);
    \path[fill=blue!35, draw=blue!65, opacity=0.62]
      (0,0)
      .. controls (0.7,0.02) and (1.0,0.22) .. (1.7,0.45)
      .. controls (2.2,0.62) and (2.7,0.42) .. (3.1,0.08)
      .. controls (3.5,-0.02) and (4.0,-0.02) .. (4.4,0.08)
      .. controls (4.9,0.18) and (5.6,0.50) .. (5.8,0.58)
      .. controls (6.25,0.50) and (6.6,0.02) .. (7,0)
      -- cycle;
    \path[fill=orange!45, draw=orange!75, opacity=0.68]
      (0.7,0)
      .. controls (1.1,0.00) and (1.35,0.35) .. (1.7,0.96)
      .. controls (1.95,1.25) and (2.25,1.25) .. (2.5,0.96)
      .. controls (2.85,0.35) and (3.1,0.00) .. (3.5,0)
      -- cycle;
    \path[fill=orange!45, draw=orange!75, opacity=0.68]
      (4.7,0)
      .. controls (5.05,0.00) and (5.25,0.35) .. (5.8,1.02)
      .. controls (6.35,0.35) and (6.55,0.00) .. (6.9,0)
      -- cycle;
  \end{tikzpicture}
};

\node[method, fill=rlcolor, right=24mm of power] (rl) {%
  \textbf{RL}\\[2pt]
  KL-regularized RL\\[-1pt]
  with self-reward\\[2pt]
  \begin{tikzpicture}[x=1.75mm,y=7mm,baseline=-0.6ex]
    \draw[black!25] (-0.3,0) -- (8.3,0);
    \foreach \x/\h in {0/0.24,1/0.55,2/0.34}
      \fill[blue!25] (\x-0.22,0) rectangle (\x+0.22,\h);
    \draw[->, blue!70!black, thick] (3.0,0.35) -- (5.0,0.35);
    \foreach \x/\h in {6/0.12,7/0.92,8/0.22}
      \fill[blue!55] (\x-0.22,0) rectangle (\x+0.22,\h);
  \end{tikzpicture}
};

\node[
  method,
  fill=distillcolor,
  below=7mm of power,
  text width=56mm,
  minimum height=13mm,
] (distill) {%
  \textbf{Self-distillation}\\[2pt]
  Power self-distillation: SFT on samples $y\sim\pi_\alpha$\\[1pt]
  {\scriptsize sharpening: Prop.~\ref{prop:teacher-sft}\quad true reward: Prop.~\ref{prop:Rprime-cov}}
};

\draw[edge] (sampling) -- node[reflabel, sloped, below] {limits of token-level\\approximation\\Prop.~\ref{prop:pow-vs-temp-renyi}, Prop.~\ref{prop:opt-proposal-power}} (power);

\draw[edge] (rl) -- node[reflabel, sloped, below] {closed-form optimum\\Cor.~\ref{cor:self-reward-power}} (power);

\draw[edge] (power) -- node[reflabel, right=1pt] {offline teacher data\\Alg.~\ref{alg:offline-power-distill}} (distill);

\draw[deriveedge] (rl.south) to[bend left=14]
  node[reflabel, right=1pt] {derive Alg.~\ref{alg:offline-power-distill}}
  (distill.east);

\begin{scope}[on background layer]
  \node[
    rounded corners=5pt,
    fill=gray!3,
    draw=black!18,
    inner sep=5pt,
    fit=(sampling)(rl)(distill)(power)
  ] {};
\end{scope}

\end{tikzpicture}}
  \caption{Overview of our contribution. The power distribution connects sampling, KL-regularized RL, and self-distillation: it is the target of power sampling, the closed-form optimum of self-reward RL, and the teacher distribution amortized by power self-distillation.}
  \label{fig:contribution}
\end{figure}

\section{Related work}
\label{sec:related}

\noindent\textbf{RL post-training as distribution sharpening.}
RL has become a central tool in LLM post-training, including RL from human feedback (RLHF)~\citep{ouyang2022instructgpt} and RL with verifiable rewards (RLVR)~\citep{shao2024deepseekmath,guo2025deepseek,lambert2024tulu}. However, the role of RL in reasoning remains debated.
\citet{liu2025prorl} showed that prolonged RL can outperform base models even under broad \texttt{pass@k} evaluations.
On the other hand, \citet{yue2025does} showed that under \texttt{pass@k} evaluation, RLVR often improves sampling efficiency at small \(k\) but can underperform the base model at large \(k\), suggesting that RLVR concentrates probability mass on reasoning paths already present in the base model's distribution.
Complementing this view, \citet{he2025rewarding} analyzed a degenerate rank bias in GRPO that preferentially reinforces high-probability trajectories, yielding a ``distribution sharpening'' regime where simply sampling more from the base model can be stronger under the same sample budget.
Motivated by the perspective that many RL gains resemble distribution sharpening, \citet{karan2026reasoning} proposed a training-free inference-time method that targets sharpened distributions of the base model. Their approach uses a Metropolis--Hastings sampler to approximate sequence-level power sampling and achieves reasoning improvements comparable to RL.
\citet{azizi2026power} and \citet{ji2026scalable} developed lower-latency approximations to power sampling.
We complement this line of work by showing that the power distribution targeted by these samplers is also the closed-form optimum of a self-reward KL-regularized RL objective.

\noindent\textbf{Inference-time compute and distillation to amortize inference cost.}
Recent work argues that allocating additional computation at inference time can substantially improve LLM outputs~\citep{snell2024scaling,welleck2024from}.
When an external reward or verifier is available, a common method is Best-of-\(N\), which generates \(N\) candidates and selects the one with the highest reward; this simple strategy can yield strong empirical gains~\citep{stiennon2020learning,nakano2021webgpt,touvron2023llama,pmlr-v202-gao23h,eisenstein2023helping,mudgal2024controlled}.
To amortize the inference cost of Best-of-\(N\), several works characterized the distribution induced by Best-of-\(N\) selection and proposed to distill this distribution into a single policy~\citep{gui2024bonbon,amini2025variational,sessa2025bond,yang2025fasterwind}.
In contrast to these reward-based distillation methods, we derive a self-distillation objective that amortizes power sampling itself, using only samples from the base model's power distribution.

\noindent\textbf{Self-improvement without external rewards.}
A growing number of empirical studies suggest that language models can improve without relying on external rewards or human-provided labels, using self-generated data and intrinsic training signals. \citet{huang2023large,wang2023self} curated model-generated solutions or instructions and then fine-tuned on them.
Several works perform RL using internal feedback alone, such as entropy minimization objectives~\citep{prabhudesai2025maximizing} or confidence as the reward~\citep{zhao2026learning}. Even randomly assigned rewards can improve performance~\citep{shao2025spurious}.
\citet{huang2025selfimprovement} formalized LLM self-improvement as distribution sharpening and analyzed algorithms motivated by SFT and KL-regularized RL.
Building on this sharpening view, we show that the model's sequence-level log-probabilities induce the power distribution through KL-regularized RL, and that distilling this distribution can sharpen the model without external rewards.

\section{Preliminaries}
\label{sec:prelim}
\noindent\textbf{Notation.}
Let \(\mathcal{X}\) denote the space of prompts and let \(\mu\in\Delta(\mathcal{X})\) denote a distribution over prompts.
We consider completions of length \(T\ge 1\) over a finite vocabulary \(\mathcal{V}\), and write \(\mathcal{Y}:=\mathcal{V}^{T}\) for the completion space.
The base model is a policy \(\pi:\mathcal{X}\to\Delta(\mathcal{Y})\) and we write \(\pi(\cdot\mid x)\) for the conditional distribution of \(y\) given \(x\).
With \(y_{<t}:=(y_{1},\dots,y_{t-1})\), we use the autoregressive factorization \(\pi(y\mid x)=\prod_{t=1}^{T}\pi(y_{t}\mid x,y_{<t})\).
We write \(a\lesssim b\) to mean \(a=O(b)\) and \(\pi(S\mid x):=\sum_{y\in S}\pi(y\mid x)\) for a set \(S\subseteq\mathcal{Y}\).

\noindent\textbf{Self-improvement.}
Language models have been shown to be capable of self-improvement, improving their own performance without external rewards~\citep{huang2023large,wang2023self,prabhudesai2025maximizing,zhao2026learning}.
This phenomenon is counterintuitive and appears to contradict the data-processing inequality, which states that mutual information is non-increasing under further processing of random variables~\citep{cover1999elements}.
\citet{huang2025selfimprovement} reconcile these observations by interpreting improvements as \emph{computational}, not statistical: self-improvement sharpens the distribution so that sampling a near-optimal solution becomes easier.
This perspective connects to classical trade-offs between sampling and optimization in theoretical computer science~\citep{kirkpatrick1983optimization,lovasz2006fast}.

Formally, define the \emph{self-reward} as the log-likelihood
\begin{equation}
r_{\mathrm{self}}(x, y;\pi) = \log \pi(y\mid x) \label{eq:rself}
\end{equation}
 and let the corresponding maximizer set be
 \begin{equation}
     \bm{y}^\star(x):= \argmax_{y\in\mathcal{Y}}r_{\mathrm{self}}(x,y;\pi).\label{eq:ystar}
 \end{equation}
Given \((\epsilon,\delta)\in(0,1)^2\), a policy \(\widehat{\pi}\) is \((\epsilon,\delta)\)-sharpened relative to \(\pi\) if the following holds:
\begin{equation}
\label{eq:sharpened}
\mathbb{P}_{x\sim\mu}\Big[\widehat{\pi}\big(\bm{y}^\star(x)\mid x\big)\ge 1-\delta\Big]\ge 1-\epsilon.
\end{equation}
\citet{huang2025selfimprovement} analyze the sample complexity of achieving \((\epsilon,\delta)\)-sharpening when \(\pi\) is accessed only through conditional draws \(y\sim\pi(\cdot\mid x)\) and likelihood evaluations \(\pi(y\mid x)\), for supervised fine-tuning on Best-of-\(N\) targets sampled from \(\pi\) and for KL-regularized RL objectives driven by \(r_{\mathrm{self}}(x,y;\pi)\).

\noindent\textbf{Power distribution.}
Recent analyses of RL suggest that empirical reasoning gains resemble \emph{distribution sharpening}, where probability mass concentrates on trajectories already well supported under the base model~\citep{yue2025does,he2025rewarding}.
Motivated by this view, \citet{karan2026reasoning} target inference-time sampling from the \emph{power distribution} induced by the base model.

\begin{definition}[Power distribution]
\label{def:power-distribution}
With a policy \(\pi:\mathcal{X}\to\Delta(\mathcal{Y})\) and an exponent \(\alpha>1\), we define the \emph{power distribution} induced by \(\pi\) as
\begin{equation}
\label{eq:power-distribution}
\pi_\alpha(y\mid x)
:=
\frac{\pi(y\mid x)^{\alpha}}{\sum_{y'\in\mathcal{Y}}\pi(y'\mid x)^{\alpha}}.
\end{equation}
\end{definition}
Exact sampling from Eq.~\eqref{eq:power-distribution} is intractable at scale.
\citet{karan2026reasoning} therefore propose a Metropolis--Hastings (MH) procedure that achieves reasoning accuracy competitive with strong RL post-training~\citep{shao2024deepseekmath,guo2025deepseek}, without further training.
Lower-latency approximations have subsequently been proposed~\citep{azizi2026power,ji2026scalable}, but these methods still use substantially more inference-time compute than standard autoregressive sampling.

\section{Approximating power sampling requires sequence-level information}
\label{sec:sampling-limits}
In this section, we begin from the sampling perspective.
We ask whether the power distribution \(\pi_\alpha\) can be reproduced by inexpensive inference-time approximations, focusing on two natural local inference-time procedures: (i) per-token temperature scaling (\Cref{sec:likelihood-mode-vanish}) and (ii) sequential importance sampling (SIS) with a one-step proposal (\Cref{subsec:opt-proposal-power}).
In both cases, the gap to \(\pi_\alpha\) is governed by sequence-level information that the local approximations do not access, showing why cheap inference-time approximations are structurally difficult and motivating the RL and self-distillation perspectives in \Cref{sec:rl-distillation}.

\subsection{Comparison to per-token temperature scaling}
\label{sec:likelihood-mode-vanish}

A natural way to locally approximate \(\pi_\alpha\) is to apply the same power transformation at the token level during decoding.
For \(s\in\mathcal{V}\), define the next-token distribution obtained by per-token temperature scaling as
\begin{equation}
\label{eq:temp-marginal-token}
\pi_{\mathrm{temp},\alpha}(y_{t}=s \mid x, y_{<t})
:=
\frac{\pi(y_{t}=s \mid x, y_{<t})^{\alpha}}{\sum_{s'\in\mathcal{V}}\pi(y_{t}=s' \mid x, y_{<t})^{\alpha}}.
\end{equation}
In contrast, the power distribution in Eq.~\eqref{eq:power-distribution} is, more precisely, the \emph{sequence-level} power distribution \(\pi_\alpha(\cdot\mid x)\propto \pi(\cdot\mid x)^\alpha\), whose next-token conditional we denote by \(\pi_{\mathrm{pow},\alpha}\):
\begin{equation}
\label{eq:pow-marginal-token}
\pi_{\mathrm{pow},\alpha}(y_{t}=s \mid x, y_{<t})
:=
\frac{\sum_{y_{t+1:T}\in\mathcal{V}^{T-t}}\pi(y_{<t},\,s,\,y_{t+1:T}\mid x)^{\alpha}}{\sum_{y_{t:T}\in\mathcal{V}^{T-t+1}}\pi(y_{<t},\,y_{t:T}\mid x)^{\alpha}}.
\end{equation}
We show that for arbitrary suffix distributions, the entire odds-ratio gap between Eqs.~\eqref{eq:temp-marginal-token} and~\eqref{eq:pow-marginal-token} is controlled by the Rényi entropy of the suffix.

\begin{proposition}[Power vs.\ temperature odds ratios via suffix Rényi entropies]
\label{prop:pow-vs-temp-renyi}
For \(\alpha>1\), a prompt \(x\), a prefix \(y_{<t}\), and \(a\in\mathcal{V}\), let \(q_{t,a}\) denote the conditional distribution of the suffix \(Y_{t+1:T}\) under the base model,
\begin{equation}
q_{t,a}(y_{t+1:T}):=\pi(y_{t+1:T}\mid x,y_{<t},y_t=a).
\end{equation}
For a distribution \(p\) on a finite set, define the Rényi entropy of order \(\alpha\) as
\(
H_{\alpha}(p):=1/(1-\alpha)\log\sum_{z} p(z)^{\alpha}.
\)
Then for any \(a,b\in\mathcal{V}\) such that \(\pi(y_{t}=a\mid x, y_{<t})>0,\pi(y_{t}=b\mid x, y_{<t})>0\), the ratio of next-token odds under \(\pi_{\mathrm{pow},\alpha}\) versus \(\pi_{\mathrm{temp},\alpha}\) satisfies
\begin{equation}
\label{eq:pow-temp-ratio-pair}
\frac{\pi_{\mathrm{pow},\alpha}(y_{t}=a \mid x, y_{<t})}{\pi_{\mathrm{pow},\alpha}(y_{t}=b \mid x, y_{<t})}
\bigg/
\frac{\pi_{\mathrm{temp},\alpha}(y_{t}=a \mid x, y_{<t})}{\pi_{\mathrm{temp},\alpha}(y_{t}=b \mid x, y_{<t})}
=
\exp\!\bigl((1-\alpha)\,(H_{\alpha}(q_{t,a})-H_{\alpha}(q_{t,b}))\bigr).
\end{equation}
\end{proposition}
We have \(1-\alpha<0\), so Eq.~\eqref{eq:pow-temp-ratio-pair}
implies that, among next-token candidates \(a\) with comparable
values of \(\pi(y_t=a\mid x,y_{<t})\), those for which \(q_{t,a}\)
has larger Rényi entropy are relatively downweighted
under \(\pi_{\mathrm{pow},\alpha}\) compared to \(\pi_{\mathrm{temp},\alpha}\).
Thus, compared with per-token temperature scaling, sequence-level power sharpening favors continuations whose suffix distributions under \(\pi\) are more peaked, i.e., have lower Rényi entropy.

\noindent\textbf{Comparison to \citet{karan2026reasoning}.}
\citet{karan2026reasoning} also studied the gap between per-token temperature and sequence-level power sampling, and formalized it in the special case of two extreme tokens (positive vs.\ negative pivotal tokens; their Example~1 and Proposition~3).
Our result enables a quantitative comparison for any two next-token candidates.

\Cref{prop:pow-vs-temp-renyi} suggests that matching the next-token distribution induced by sequence-level power sampling at a step requires information about the suffix distributions following each candidate token.

\subsection{Variance-minimizing one-step proposals for sequential power sampling}
\label{subsec:opt-proposal-power}

Beyond marginal token distributions, we turn to sequential importance sampling (SIS) targeting~\(\pi_\alpha\), where a basic design goal is to stabilize incremental importance weights.
Proposition~3.3 of \citet{zhao2024probabilistic} identifies the unique one-step variance-minimizing proposal in a general SIS setup, and we apply it to the power distribution~\(\pi_\alpha\).

Fix a prompt \(x\in\mathcal{X}\).
Define the unnormalized power mass \(\tilde\pi_\alpha(y_{1:T}):=\pi(y_{1:T}\mid x)^{\alpha}\) and, for \(t=0,\dots,T\), the prefix totals
\begin{equation}
\label{eq:power-prefix-marg}
\tilde\pi_{\alpha,t}(y_{1:t})
\;:=\;
\sum_{y_{t+1:T}\in\mathcal{V}^{T-t}}
\tilde\pi_\alpha(y_{1:T}),
\end{equation}
where for \(t=0\) the prefix is empty.
Let \(Z_\alpha:=\sum_{y_{1:T}\in\mathcal{V}^T}\tilde\pi_\alpha(y_{1:T})\) be the normalizing constant, so that \(\tilde\pi_{\alpha,0}=Z_\alpha\), and let \(\pi_\alpha(\cdot\mid x)\) be the normalized power distribution on \(\mathcal{V}^T\) from Eq.~\eqref{eq:power-distribution}.
For \(t\ge 1\), write \(\pi_\alpha(y_{1:t}\mid x)\) for the \emph{prefix marginal} obtained by summing \(\pi_\alpha(y_{1:T}\mid x)\) over \(y_{t+1:T}\); then \(\pi_\alpha(y_{1:t}\mid x)=\tilde\pi_{\alpha,t}(y_{1:t})/Z_\alpha\).

Consider extending a fixed prefix \(y_{<t}\) by one token \(Y_t\sim q(\cdot\mid x,y_{<t})\) in one step of SIS (or SMC without resampling), while keeping the global target \(\pi_\alpha(\cdot\mid x)\) on \(\mathcal{V}^T\).
Define the \emph{incremental importance weight}~\citep{chopin2020introduction}
\begin{equation}
\label{eq:inc-weight-def}
W_t
\;:=\;
\frac{\pi_\alpha(y_{<t},Y_t\mid x)}{\pi_\alpha(y_{<t}\mid x)\,q(Y_t\mid x,y_{<t})},
\end{equation}
where we condition on \(y_{<t}\) with \(\pi_\alpha(y_{<t}\mid x)>0\), and \(\mathrm{Var}[W_t]\) denotes variance under \(Y_t\sim q(\cdot\mid x,y_{<t})\).
The next proposition shows the unique proposal \(q(\cdot\mid x,y_{<t})\) that minimizes \(\mathrm{Var}[W_t]\) at such a prefix.

\begin{proposition}[Variance-minimizing one-step proposal at prefix $y_{<t}$]
\label{prop:opt-proposal-power}
In the setting above, fix \(y_{<t}\) with \(\pi_\alpha(y_{<t}\mid x)>0\).
Among all proposals \(q(\cdot\mid x,y_{<t})\) on \(\mathcal{V}\), the unique minimizer of \(\mathrm{Var}[W_t]\) is
\begin{equation}
\label{eq:opt-q-power}
q_t^\star(y_t\mid x,y_{<t})
\;:=\;
\frac{\tilde\pi_{\alpha,t}(y_{1:t})}{\tilde\pi_{\alpha,t-1}(y_{<t})}
\;=\;
\pi_\alpha\bigl(y_t\mid x,y_{<t}\bigr),
\end{equation}
where \(y_{1:t}=(y_{<t},y_t)\); the right-hand side equals the next-token conditional under \(\pi_\alpha(\cdot\mid x)\).
\end{proposition}

\Cref{prop:opt-proposal-power} implies that minimizing the \emph{local} one-step variance of the incremental weight forces the proposal to coincide with the next-token conditional in Eq.~\eqref{eq:opt-q-power}, which itself depends on the prefix totals \(\tilde\pi_{\alpha,t}\) summed over all suffixes.
In particular, proposals that modify only the base next-token conditional cannot in general equal the unique minimizer in Eq.~\eqref{eq:opt-q-power}.
The proof and SIS background are in \Cref{app:sis-background}.

\noindent\textbf{Implication.}
\Cref{prop:pow-vs-temp-renyi,prop:opt-proposal-power} indicate that inexpensive one-step approximations cannot reproduce \(\pi_\alpha\) without sequence-level information, leaving inference-time approximation of \(\pi_\alpha\) structurally expensive.
This aligns with prior work that expends additional inference-time compute~\citep{karan2026reasoning,azizi2026power,ji2026scalable} to approximate the power distribution.

\section{From self-reward RL to power self-distillation}
\label{sec:rl-distillation}

\Cref{sec:sampling-limits} shows that \(\pi_\alpha\) is structurally expensive to approximate by sampling at inference time.
In this section, we take the complementary
view that \(\pi_\alpha\) also connects RL and self-distillation, allowing us to shift the cost to offline
training.
\Cref{sec:power-rl} identifies \(\pi_\alpha\) as
the closed-form optimum of a KL-regularized RL
objective with self-reward.
\Cref{sec:power-distillation} uses this
identification to derive an offline
self-distillation algorithm from that RL objective.
\Cref{sec:true-reward} then analyzes what the
resulting distilled model achieves: a self-reward
sharpening guarantee, a covariance characterization
of the ideal power teacher's true-reward gain, and a
bound transferring this gain to the student when the
forward-KL distillation error is small.

\subsection{Power distribution as the optimum of self-reward RL}
\label{sec:power-rl}
Let \(q:\mathcal{X}\to\Delta(\mathcal{Y})\) be a candidate policy, and consider the KL-regularized RL objective with reward \(r\)~\citep{ouyang2022instructgpt,guo2025deepseek}
\begin{equation}
\label{eq:soft-rl-obj}
J_\beta(q;\pi,r)
:=
\mathbb{E}_{x\sim\mu}\Big[
\mathbb{E}_{y\sim q(\cdot\mid x)}\big[r(x,y)\big]
-\beta\,D_{\mathrm{KL}}\!\big(q(\cdot\mid x)\,\|\,\pi(\cdot\mid x)\big)
\Big]
\end{equation}
with \(\beta>0\).
By the standard closed-form solution of KL-regularized RL~\citep{levine2018reinforcement}, the unique maximizer of Eq.~\eqref{eq:soft-rl-obj} is the reward-tilted distribution
\begin{equation}
\label{eq:tilted-policy}
\pi_\beta^\star(y\mid x)
:=
\frac{\pi(y\mid x)\exp\!\big(\beta^{-1}r(x,y)\big)}{Z_r(x)},
\qquad
Z_r(x):=\sum_{y'\in\mathcal{Y}}\pi(y'\mid x)\exp\!\big(\beta^{-1}r(x,y')\big).
\end{equation}
We restate this as \Cref{prop:kl-regularized-tilt} in \Cref{app:proof-kl-regularized-tilt} and include a proof for completeness.
Specializing the reward in Eq.~\eqref{eq:tilted-policy} to the self-reward~\(r_{\mathrm{self}}\) in Eq.~\eqref{eq:rself} yields the power distribution.

\begin{corollary}[Self-reward tilt equals the power distribution]
\label{cor:self-reward-power}
Suppose \(r(x,y)=r_{\mathrm{self}}(x,y;\pi)=\log\pi(y\mid x)\) as in Eq.~\eqref{eq:rself}.
Then the optimizer \(\pi_\beta^\star\) in Eq.~\eqref{eq:tilted-policy} equals the power distribution \(\pi_\alpha\) in Eq.~\eqref{eq:power-distribution}:
\begin{equation}
\label{eq:power-sampling-optimal}
\pi_\beta^\star(\cdot\mid x)=\pi_\alpha(\cdot\mid x),
\qquad
\alpha:=1+\beta^{-1}>1.
\end{equation}
\end{corollary}

This identification connects power sampling and self-improvement RL: the inference-time target of \citet{karan2026reasoning} coincides with the closed-form optimum of the KL-regularized self-reward objective studied in \citet{huang2025selfimprovement}, namely \(\pi_\alpha\).

\subsection{Deriving power self-distillation}
\label{sec:power-distillation}

We now derive a self-distillation procedure from the RL objective without requiring the deployed model to sample from \(\pi_\alpha\) at inference time.

\noindent\textbf{RL objective as reverse and then forward KL to \(\pi_\alpha\).}
With \(r=r_{\mathrm{self}}\) and \(\alpha=1+\beta^{-1}\), the inner objective in Eq.~\eqref{eq:soft-rl-obj} can be rewritten for each \(x\) as
\begin{equation}
\label{eq:rl-equals-kl}
\scalebox{0.97}{$
\mathbb{E}_{y\sim q(\cdot\mid x)}[r_{\mathrm{self}}(x,y)]
-\beta D_{\mathrm{KL}}\!\big(q(\cdot\mid x)\|\pi(\cdot\mid x)\big)
=
-\beta D_{\mathrm{KL}}\!\big(q(\cdot\mid x)\|\pi_{\alpha}(\cdot\mid x)\big)
+\beta\log Z_r(x),
$}
\end{equation}
with the same partition function \(Z_r(x)\) as in Eq.~\eqref{eq:tilted-policy}, which does not depend on~\(q\).
Thus, for each prompt~\(x\), maximizing \(J_\beta(q;\pi,r_{\mathrm{self}})\) over unconstrained~\(q\) is equivalent to minimizing the reverse KL divergence \(D_{\mathrm{KL}}\!\big(q(\cdot\mid x)\,\|\,\pi_{\alpha}(\cdot\mid x)\big)\), with unique minimizer \(q(\cdot\mid x)=\pi_\alpha(\cdot\mid x)\).

The left-hand side of Eq.~\eqref{eq:rl-equals-kl} is the original self-reward KL-regularized RL objective, whose direct optimization requires online policy optimization and is typically more challenging than supervised training.
The reverse-KL form on the right-hand side is also an expectation under~\(q\), so direct optimization still requires on-policy samples from the learner. This can be viewed as an online distillation procedure.
However, this approach requires generation inside the training loop and substantially increases training time.

We therefore convert the objective into an offline distillation surrogate that shares the same target distribution~\(\pi_\alpha\), by minimizing the forward KL from the teacher distribution to the student,
\(
D_{\mathrm{KL}}\!\big(\pi_\alpha(\cdot\mid x)\,\|\,q(\cdot\mid x)\big)
\).
The population minimizer over~\(q\) is still~\(\pi_\alpha\), so this surrogate preserves the same target distribution while enabling offline maximum-likelihood training on teacher samples.
This forward-KL surrogate mirrors the reward-augmented maximum-likelihood method of \citet{norouzi2016reward}, who also exchanged the reverse KL appearing in entropy-regularized RL for a forward KL.

\noindent\textbf{Forward KL yields MLE on teacher samples.}
Expanding the forward KL training objective gives
\begin{align}
\label{eq:forward-kl-to-mle}
\mathbb{E}_{x\sim\mu}\!\left[D_{\mathrm{KL}}\!\big(\pi_\alpha(\cdot\mid x)\,\|\,q(\cdot\mid x)\big)\right]
&=
\mathbb{E}_{x,y\sim \mu, \pi_\alpha}\big[\log \pi_\alpha(y\mid x)\big]
-
\mathbb{E}_{x,y\sim \mu, \pi_\alpha}\big[\log q(y\mid x)\big],
\end{align}
where \(x,y\sim\mu, \pi_\alpha\) abbreviates \(x\sim\mu\) and \(y\sim\pi_\alpha(\cdot\mid x)\).
The first term on the right-hand side of Eq.~\eqref{eq:forward-kl-to-mle} does not depend on~\(q\), so minimizing the population forward KL is equivalent to maximizing the expected log-likelihood \(\mathbb{E}_{x,y\sim\mu,\pi_\alpha}\big[\log q(y\mid x)\big]\).
In practice we form an empirical objective by drawing i.i.d.\ pairs \(D=\{(x_i,y_i)\}_{i=1}^{n}\) with \(x_i\sim\mu\) and \(y_i\sim\pi_\alpha(\cdot\mid x_i)\). Let \(\Pi_\alpha\) be a student policy class that depends on \(\alpha\). We solve the following maximum likelihood estimate (MLE) problem:
\begin{equation}
\label{eq:power-distill-mle}
\widehat\pi\in\arg\max_{q\in\Pi_\alpha}\sum_{i=1}^{n}\log q(y_i\mid x_i).
\end{equation}
This procedure uses only offline completions sampled from the power distribution~\(\pi_\alpha\) derived from the base policy~\(\pi\), and it does not rely on any external reward labels, so it is an instance of self-distillation.
We refer to it as \emph{power self-distillation} and summarize it in \Cref{alg:offline-power-distill}.
In practice we run teacher inference once, store \(D=\{(x_i,y_i)\}_{i=1}^{n}\), and then train the student with standard supervised fine-tuning on~\(D\).
Separating teacher generation from student training simplifies implementation and enables dataset reuse.
In \Cref{sec_app:additional-analyses}, we provide a complementary KL-gap analysis.

\subsection{Sharpening and true reward under power self-distillation}
\label{sec:true-reward}

In this subsection, we analyze three complementary aspects of power self-distillation: (i) \Cref{prop:teacher-sft} bounds the extent to which \(\widehat\pi\) sharpens the self-reward, in the sense of \citet{huang2025selfimprovement}; (ii) \Cref{prop:Rprime-cov} shows that the local rate at which sharpening changes a true reward is determined by the covariance \(\mathrm{Cov}_{\pi_\alpha}(r^\star, r_{\mathrm{self}})\); and (iii) \Cref{prop:gain-transfer} shows how the resulting power-teacher gain transfers to an imperfect distilled student.

\noindent\textbf{(i) Self-reward sharpening of the distilled model.}
\citet{huang2025selfimprovement} formalize self-improvement via \((\epsilon,\delta)\)-sharpening relative to~\(\pi\) as in Eq.~\eqref{eq:sharpened}.
The next proposition bounds how well the MLE in Eq.~\eqref{eq:power-distill-mle} concentrates on the self-reward maximizer set~\(\bm{y}^\star(x)\) in Eq.~\eqref{eq:ystar}.

\begin{proposition}[Power self-distillation and sharpening]
\label{prop:teacher-sft}
Fix \(\alpha>1\) and \(\rho,\delta\in(0,1)\). Suppose \(\pi_\alpha\in\Pi_\alpha\) and there exists a constant \(M<\infty\), independent of \(\alpha\), such that \(|\Pi_\alpha|\le M\).
Let \(D=\{(x_i,y_i)\}_{i=1}^{n}\) be i.i.d.\ samples with \(x_i\sim\mu, y_i\sim\pi_\alpha(\cdot\mid x_i)\), and let
\(\widehat\pi\in\arg\max_{q\in\Pi_\alpha}\sum_{i=1}^{n}\log q(y_i\mid x_i)\) be an MLE.
Then with probability at least \(1-\rho\) over \(D\),
\begin{equation}
\label{eq:teacher-distillation}
\mathbb{P}_{x\sim\mu}\big[\widehat\pi(\bm{y}^\star(x)\mid x)\le 1-\delta\big]
\ \lesssim\
\frac{\log(M\rho^{-1})}{\delta n}
\;+\;
\mathbb{P}_{x\sim\mu}\big[\pi_\alpha(\bm{y}^\star(x)\mid x)\le 1-\tfrac{\delta}{2}\big].
\end{equation}
In particular, the right-hand side of Eq.~\eqref{eq:teacher-distillation} converges to $0$ as \(n\to\infty\) and \(\alpha\to\infty\).
\end{proposition}

Thus, for sufficiently large \(n\) and \(\alpha\), power self-distillation can achieve \((\epsilon,\delta)\)-sharpening in the sense of Eq.~\eqref{eq:sharpened}.
The proof is in \Cref{app:proof-teacher-sft}.

\noindent\textbf{(ii) When does sharpening also improve a different true reward?}
\Cref{prop:teacher-sft} guarantees concentration on the self-reward maximizer set, but evaluation is typically governed by a different true reward (e.g., correctness).
Let \(r^\star:\mathcal{X}\times \mathcal{Y}\to \mathbb{R}\) denote this true reward and define, for fixed \(x\in\mathcal{X}\),
\begin{equation}
R(\alpha; x):=\mathbb{E}_{y\sim\pi_\alpha(\cdot\mid x)}\big[r^\star(x,y)\big].
\end{equation}
The next proposition characterizes how \(R(\alpha;x)\) changes with \(\alpha\).

\begin{proposition}[Covariance form of \(\partial_\alpha R(\alpha;x)\)]
\label{prop:Rprime-cov}
For any \(\alpha>0\) and any fixed \(x\in\mathcal{X}\),
\begin{equation}
\label{eq:Rprime-cov}
\frac{\partial}{\partial \alpha}R(\alpha;x)
=
\mathrm{Cov}_{y\sim\pi_\alpha(\cdot\mid x)}\!\big(r^\star(x,y),r_{\mathrm{self}}(x,y)\big),
\end{equation}
where covariances are over the support of \(\pi(\cdot\mid x)\), on which \(\log \pi(\cdot \mid x)\) is finite.
In particular, if for some \(b(x)\in\mathbb{R}\) and \(c(x)>0\),
\begin{equation}
\label{eq:rstar-affine-logpi}
r^\star(x,y)=c(x)\,r_{\mathrm{self}}(x,y)+b(x)\qquad\forall y\in\mathcal{Y},
\end{equation}
then \(R(\alpha;x)\) is non-decreasing in \(\alpha\):
\begin{equation}
\frac{\partial}{\partial \alpha}R(\alpha;x)
=
c(x)\,\mathrm{Var}_{y\sim\pi_\alpha(\cdot\mid x)}\!\big(r_{\mathrm{self}}(x,y)\big)\ge 0.
\end{equation}
\end{proposition}

The proof is in \Cref{app:proof-Rprime-cov}.
\Cref{prop:Rprime-cov} states that \(\partial_\alpha R(\alpha;x)\) equals the covariance between \(r^\star\) and \(r_{\mathrm{self}}\) under \(\pi_\alpha(\cdot\mid x)\), so whether increasing \(\alpha\) improves the true reward is determined exactly by \(\mathrm{Cov}_{y\sim\pi_\alpha(\cdot\mid x)}\!\big(r^\star(x,y),r_{\mathrm{self}}(x,y)\big)\).
In particular, when \(r^\star=r_{\mathrm{self}}\), this covariance reduces to \(\mathrm{Var}_{y\sim\pi_\alpha(\cdot\mid x)}\!\big(\log\pi(y\mid x)\big)\ge 0\), so \(R(\alpha;x)\) is non-decreasing in \(\alpha\).

\noindent\textbf{(iii) Transfer to an imperfect distilled student.}
The covariance identity characterizes the ideal teacher \(\pi_\alpha\), whereas power self-distillation produces a student \(q\) trained to approximate \(\pi_\alpha\).
Define the population true reward
\[
\mathcal{R}(q)
:=
\mathbb{E}_{x\sim\mu}
\mathbb{E}_{y\sim q(\cdot\mid x)}[r^\star(x,y)]
\]
and the integrated covariance gain
\[
\Gamma_\alpha
:=
\mathbb{E}_{x\sim\mu}
\left[
\int_1^\alpha
\mathrm{Cov}_{y\sim\pi_s(\cdot\mid x)}
\bigl(r^\star(x,y),r_{\mathrm{self}}(x,y)\bigr)\,\mathrm{d}s
\right].
\]
By \Cref{prop:Rprime-cov} and the fundamental theorem of calculus, \(\Gamma_\alpha=\mathcal{R}(\pi_\alpha)-\mathcal{R}(\pi)\).
The next proposition decomposes the student's true-reward gain into this ideal power-teacher gain minus a student approximation term.

\begin{proposition}[Transfer of power-teacher gain]
\label{prop:gain-transfer}
Assume \(r_{\min}\le r^\star(x,y)\le r_{\max}\) for all \((x,y)\), and let \(\Delta_r:=r_{\max}-r_{\min}\).
For distributions \(p,q\), write \(\mathrm{TV}(p,q):=\sup_A |p(A)-q(A)|\).
Then for any policy \(q\),
\begin{equation}
\label{eq:tv-transfer-main}
\mathcal{R}(q)-\mathcal{R}(\pi)
\ge
\Gamma_\alpha
-
\Delta_r\,
\mathbb{E}_{x\sim\mu}
\!\left[
\mathrm{TV}\bigl(q(\cdot\mid x),\pi_\alpha(\cdot\mid x)\bigr)
\right].
\end{equation}
Consequently, if \(\mathbb{E}_{x\sim\mu}D_{\mathrm{KL}}(\pi_\alpha(\cdot\mid x)\|q(\cdot\mid x))<\infty\), then
\begin{equation}
\label{eq:forward-transfer-main}
\mathcal{R}(q)-\mathcal{R}(\pi)
\ge
\Gamma_\alpha
-
\Delta_r
\sqrt{
\frac{
\mathbb{E}_{x\sim\mu}
D_{\mathrm{KL}}\!\bigl(\pi_\alpha(\cdot\mid x)\|q(\cdot\mid x)\bigr)
}{2}
}.
\end{equation}
\end{proposition}

The proof is in \Cref{app:proof-gain-transfer}.
\Cref{prop:gain-transfer} directly matches the power self-distillation objective: minimizing the population forward KL to \(\pi_\alpha\) reduces the student-transfer penalty in Eq.~\eqref{eq:forward-transfer-main}.
Thus, when the integrated covariance gain \(\Gamma_\alpha\) is positive and the distilled student is close enough to the power teacher, the true-reward gain of the power distribution transfers to the student.

\begin{figure}[tbp]
  \begin{minipage}[b]{0.48\linewidth}
\centering\includegraphics[width=\linewidth]{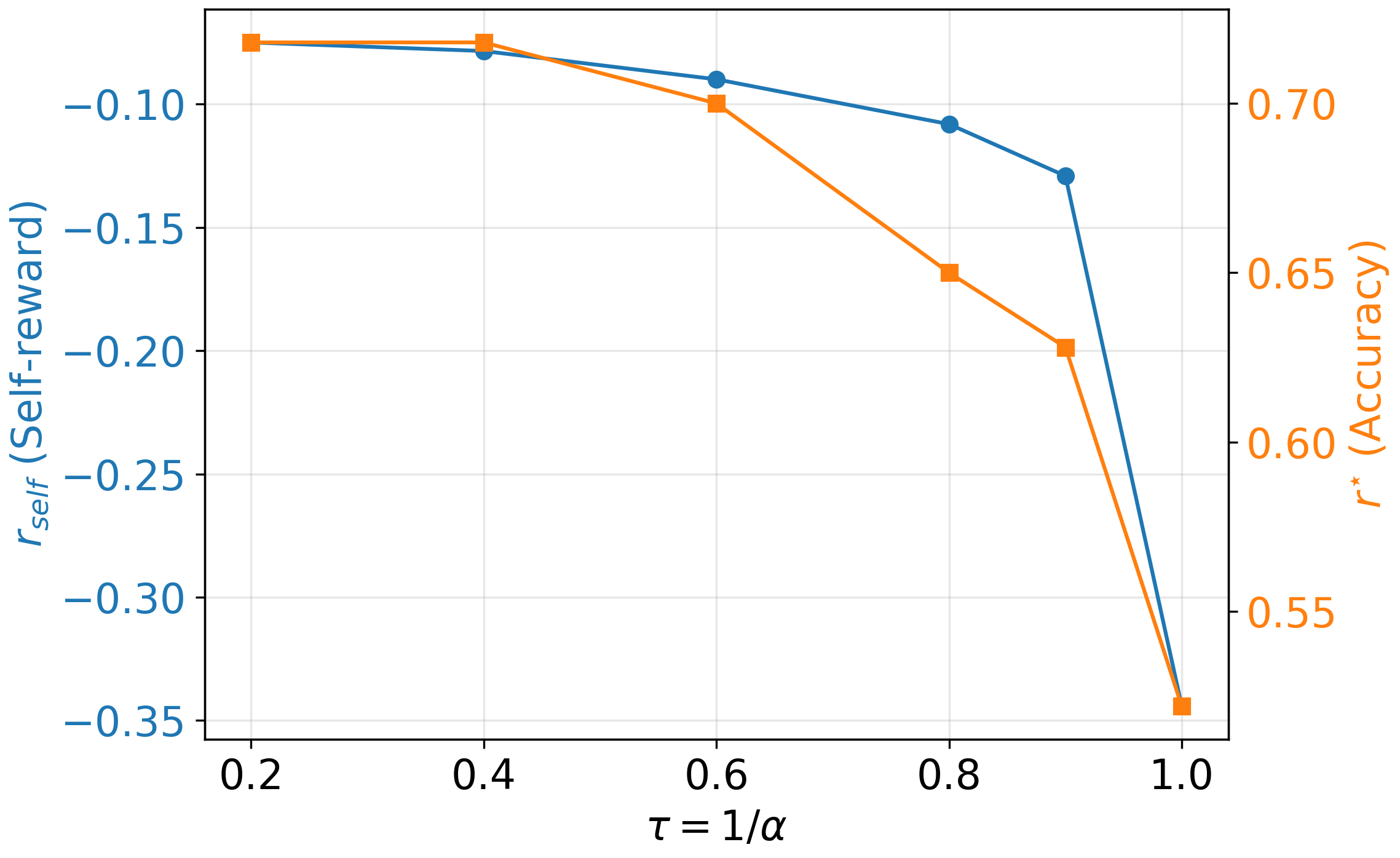}
    \caption{\textbf{Sharper power sampling raises both \(r_{\mathrm{self}}\) and \(r^\star\).} A smaller temperature parameter \(\tau=1/\alpha\) (sharper distribution) increases \(r_{\mathrm{self}}\) and true reward \(r^\star\) (accuracy) on MATH500.}
    \label{fig:tau_r}
  \end{minipage}
  \hfill
\begin{minipage}[b]{0.48\linewidth}
    \centering
    \includegraphics[width=\linewidth]{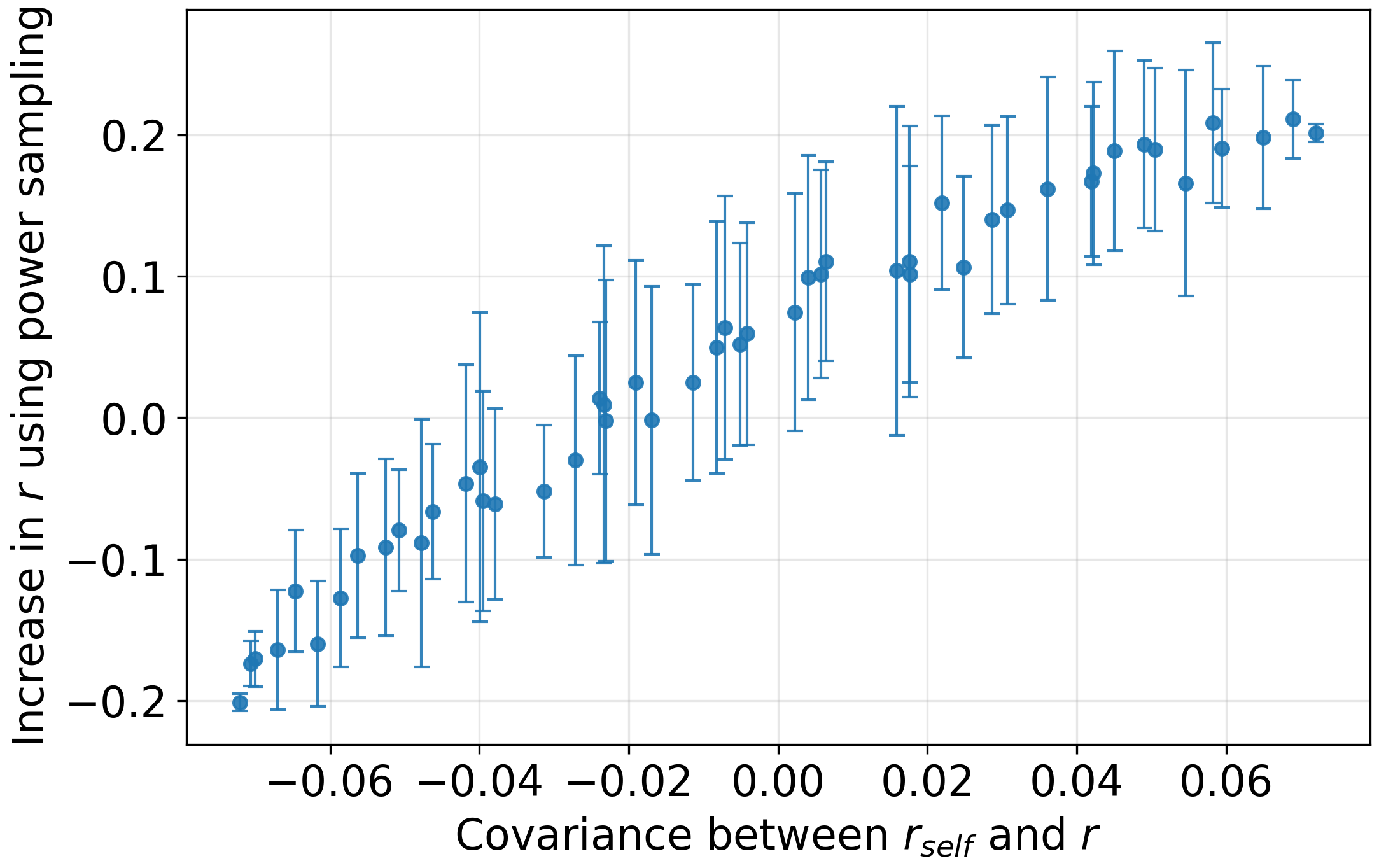}
    \caption{\textbf{Empirical illustration of \Cref{prop:Rprime-cov}.} Synthetic rewards \(r\) with varying correlation to \(r_{\mathrm{self}}\) on MATH500, with error bars denoting standard deviation.}
    \label{fig:reward_corr}
  \end{minipage}
\end{figure}

\section{Numerical evaluation}
\label{sec:numerical}
This section experimentally validates the following points.

\begin{itemize}[leftmargin=1.0em,topsep=0pt,itemsep=1pt,parsep=0pt,partopsep=0pt]
   \item (RQ1) Power sampling increases self-reward (\Cref{sec:power-rl}).
   \item (RQ2) Sharpening can improve true reward when \(r^\star\) aligns with \(r_{\mathrm{self}}\) (\Cref{sec:true-reward}).
   \item (RQ3) Power self-distillation achieves self-improvement (\Cref{sec:power-distillation,sec:true-reward}).
\end{itemize}
Detailed experimental setups are provided in \Cref{sec_app:experimental-setups}, and synthetic experiments validating \Cref{sec:sampling-limits} are shown in \Cref{sec_app:synthetic-renyi-validation,sec_app:synthetic-sis-validation}.

\subsection{Setup}
We used the Qwen2.5-Math-7B~\citep{yang2024qwen2}, Qwen2.5-7B~\citep{yang2024qwen}, and Phi-3.5-mini-instruct~\citep{abdin2024phi} models on the MATH~\citep{lightman2024lets}, HumanEval~\citep{chen2021evaluating}, MBPP~\citep{austin2021program}, and GPQA~\citep{rein2024gpqa} datasets.
In the main text, we focus on the Qwen2.5-Math-7B model on the MATH dataset, which consists of 12,500 competition-style math problems spanning seven categories.
For evaluation, we used MATH500, a selected subset of the MATH test set.
For power self-distillation (\Cref{alg:offline-power-distill}), we sampled 500 training problems from MATH, excluding those in MATH500.
We fine-tuned with LoRA adapters~\citep{hu2022lora} using the AdamW optimizer~\citep{loshchilov2017decoupled}.

For power sampling, we used the MH procedure of \citet{karan2026reasoning} (\Cref{alg:samp}) with their default hyperparameters, including \(\alpha=4.0\).
For additional baselines, we used standard autoregressive sampling (Standard) and token-wise temperature scaling (Temperature) with \(\tau=1/\alpha=0.25\), so that the token-level baseline uses the same local power exponent as power sampling.
We studied three model variants: the base model (Base), the power-distilled model (Power-distilled, \Cref{alg:offline-power-distill}), and a randomly initialized model (RandW).
RandW is a negative control for cases where likelihood is not aligned with correctness.

To study the relationship between true reward \(r^\star\) and self-reward \(r_{\mathrm{self}}\), we additionally evaluated an approach we call self-reward Best-of-\(N\): given \(N\) sampled completions \(\{y_i\}_{i=1}^N\), we selected the completion with the largest value of \(r_{\mathrm{self}}(y_i)\).
In all experiments, \(r_{\mathrm{self}}\) denotes the completion-token average log-likelihood under the corresponding base model, with prompt tokens masked out; this length normalization makes values comparable across completions.

\begin{table}[t]
\centering
\caption{True reward~$r^{\star}$ (accuracy) and self-reward~$r_{\mathrm{self}}$ for Qwen2.5-Math-7B on MATH500. The left two columns report means over all sampled completions. The right two columns report self-reward Best-of-\(N\): the completion with the largest \(r_{\mathrm{self}}\) among samples generated with different seeds. Evaluated with four seeds.}
\label{tab:logp-seed-aggregates}
\begin{tabular}{llcccc}
  \toprule
  & & \multicolumn{2}{c}{All completions} & \multicolumn{2}{c}{Self-reward Best-of-\(N\)} \\
  \cmidrule(lr){3-4} \cmidrule(lr){5-6}
  Model & Sampling & $r^{\star} (\uparrow)$ & $r_{\mathrm{self}}$ & $r^{\star} (\uparrow)$ & $r_{\mathrm{self}}$ \\
  \midrule
  \multirow{2}{*}{RandW} & Standard & $0.000 \pm 0.000$ & $-13.406 \pm 0.008$ & $0.000$ & $-13.309$ \\
   & Power & $0.000 \pm 0.000$ & $-12.370 \pm 0.014$ & $0.000$ & $-12.133$ \\
  \midrule
  \multirow{3}{*}{Base} & Standard & $0.508 \pm 0.016$ & $-0.316 \pm 0.043$ & $0.680$ & $-0.097$ \\
   & Temperature & $0.683 \pm 0.014$ & $-0.061 \pm 0.015$ & $0.756$ & $-0.036$ \\
   & Power & $0.714 \pm 0.006$ & $-0.077 \pm 0.001$ & $0.742$ & $-0.062$ \\
  \midrule
  \multirow{2}{*}{Power-distilled} & Standard & $0.643 \pm 0.025$ & $-0.089 \pm 0.005$ & $0.764$ & $-0.042$ \\
   & Temperature & $\mathbf{0.722} \pm 0.009$ & $\mathbf{-0.043} \pm 0.001$ & $\mathbf{0.768}$ & $\mathbf{-0.034}$ \\
  \bottomrule
  \end{tabular}
  \end{table}

\subsection{Results}

\noindent\textbf{Power sampling increases self-reward (RQ1).}
\Cref{tab:logp-seed-aggregates} shows mean self-reward (\(r_{\mathrm{self}}\)) and accuracy (\(r^{\star}\)) over sampled completions (left two columns).
Power sampling raises \(r_{\mathrm{self}}\) for both the base model and RandW.
Decoding with token-wise temperature also raises \(r_{\mathrm{self}}\) on the base model.

\noindent\textbf{Sharpening can improve true reward when \(r^\star\) aligns with \(r_{\mathrm{self}}\) (RQ2).}
\Cref{tab:logp-seed-aggregates} also shows that higher \(r_{\mathrm{self}}\) is typically accompanied by higher true reward \(r^\star\), except on RandW, where \(r_{\mathrm{self}}\) is not aligned with \(r^\star\).
Notably, self-reward Best-of-\(N\) yields the largest gains in \(r^\star\) across all models.

To make this point clearer, \Cref{fig:tau_r} plots the decoding temperature \(\tau=1/\alpha\) against \(r_{\mathrm{self}}\) and \(r^\star\); both quantities decrease as \(\tau\) increases (i.e., sharpening weakens).
\Cref{fig:reward_corr} uses synthetic rewards (\Cref{sec_app:experimental-setups} for details), whose correlation with \(r_{\mathrm{self}}\) ranges from positive to negative; the gain in \(r\) from power sampling grows roughly linearly with \(\mathrm{Cov}(r,r_{\mathrm{self}})\).

\noindent\textbf{Power self-distillation achieves self-improvement (RQ3).}
\Cref{tab:logp-seed-aggregates} shows that after power self-distillation, the student with temperature decoding scores higher on both \(r^\star\) and \(r_{\mathrm{self}}\) than the base model under standard sampling, temperature sampling, or power sampling.
The strongest result is obtained by combining power self-distillation with Temperature decoding.
At inference time, the student uses only autoregressive decoding (with temperature), thereby amortizing the inference cost of power sampling into offline training.

\noindent\textbf{Qualitative example.}
\Cref{tab:qualitative_math} summarizes completions on one MATH500 problem.
With token-wise temperature, the model cites irrelevant facts and concludes with a hallucinated formula, plausibly because token-wise tilting in Eq.~\eqref{eq:temp-marginal-token} does not coincide with sequence-level tilting in Eq.~\eqref{eq:pow-marginal-token}.
Power sampling instead tilts toward \(\pi_\alpha\) and is graded correct, but the completion includes plausible Python code that is never executed, and the model only mimics a reasoning pattern.
After power self-distillation, standard decoding yields the correct answer with more robust step-by-step reasoning.
The full completions are shown in~\Cref{sec_app:qualitative}.

Additional dataset--model combinations are reported in~\Cref{sec_app:other}; in each case, the distilled model outperforms the corresponding base model.

\begin{table}[t]
  \centering
  \small
  \setlength{\tabcolsep}{4pt}
  \renewcommand{\arraystretch}{1.12}
  \caption{Qualitative comparison for Qwen2.5-Math-7B on a MATH500 question: \textit{``The
  coordinates of a parallelogram are $(5,3)$, $(6,8)$, $(7,4)$,
  and $(x,y)$ with $x>7$. What is the value of $x+y$?''}}
  \label{tab:qualitative_math}
  \resizebox{\linewidth}{!}{
  \begin{tabular}{llcp{0.60\linewidth}}
    \toprule
    Model & Sampling & Correctness & Summary \\
    \midrule
    \multirow{6}{*}{Base} & Base & No&
    Reaches the correct value but lacks sound reasoning. Terminates immediately without following the required formatting constraints. \\
    & Temperature & No & Uses irrelevant mathematical properties and generates an incorrect formula, resulting in a hallucinated final answer. \\
    & Power & Yes &
    Maintains logical consistency and mathematical accuracy, but simulates a Python execution to present a non-executed solution. \\
    Power-distilled & Standard & Yes &
    Shows robust reasoning and self-correction by re-evaluating the problem when constraints are not met. \\
    \bottomrule
  \end{tabular}
  }
\end{table}

\section{Conclusion}
\label{sec:conclusion}
We showed that the power distribution bridges power sampling, self-reward KL-regularized RL, and self-distillation as the sampling target, closed-form RL optimum, and teacher distribution.
From the sampling perspective, inexpensive local approximations are structurally limited: per-token temperature scaling and variance-minimizing one-step proposals both miss sequence-level information.
From the RL perspective, the same sequence-level power distribution is the optimizer of KL-regularized RL when the reward is the model's sequence-level log-probabilities.
This identification yields power self-distillation, an offline surrogate that amortizes power sampling into supervised training on teacher samples.
Power self-distillation can achieve self-reward sharpening, while true-reward improvement is governed by \(\mathrm{Cov}_{\pi_\alpha}(r^\star,r_{\mathrm{self}})\), and the resulting power-teacher gain transfers to a distilled student up to a forward-KL approximation term.
Finally, we supported the analysis with experiments.

\noindent\textbf{Limitations.}
Self-improvement through sharpening and distillation inherits the capabilities of the base model, so gains can be small when the base is weak; improving base-model quality (e.g., pretraining) is outside our scope.
Our analysis and experiments focus on autoregressive language models over finite horizons.
We do not fully rule out pretraining-time benchmark contamination in the underlying LLMs.
Our experiments are also limited to 7B-scale and smaller models; evaluation on substantially larger models is left to future work.

\begin{ack}
This work was supported by JSPS KAKENHI Grant Number JP24H00709.
\end{ack}

\bibliographystyle{plainnat}
\bibliography{ref}

\clearpage

\appendix

\renewcommand{\thetable}{S.\arabic{table}}
\renewcommand{\thefigure}{S.\arabic{figure}}
\setcounter{table}{0}
\setcounter{figure}{0}
\renewcommand{\theHtable}{S.\arabic{table}}
\renewcommand{\theHfigure}{S.\arabic{figure}}

\begin{algorithm}[htbp]
\caption{Power self-distillation}
\label{alg:offline-power-distill}
\begin{algorithmic}[1]
\Require Base model \(\pi\); power exponent \(\alpha>0\); teacher sampler \(\textsc{TeacherSample}(x;\pi,\alpha)\) approximating \(\pi_\alpha(\cdot\mid x)\) (e.g., \Cref{alg:samp}); prompt source \(\mu\); dataset size \(n\); student model \(q_\theta\) (initialized from \(\pi\)).
\Ensure Trained student \(q_\theta\).
\State \textbf{Offline teacher sampling (data collection):}
\For{\(i=1,\dots,n\)}
    \State Sample \(x_i \sim \mu\).
    \State Sample \(y_i \sim \textsc{TeacherSample}(x_i;\pi,\alpha)\).
\EndFor
\State Store \(D \gets \{(x_i,y_i)\}_{i=1}^n\).
\State \textbf{Student training (supervised fine-tuning):}
\State Optimize completion-only NLL on \(D\):
\[
\theta \leftarrow \arg\min_{\theta}\sum_{(x,y)\in D} -\log q_\theta(y\mid x),
\]
\State where the loss is computed only on completion tokens by masking prompt tokens.
\end{algorithmic}
\end{algorithm}

\begin{algorithm}[htbp]
\caption{Power sampling using Metropolis--Hastings~\citep{karan2026reasoning} {\color{algvariant}(Power($\infty$): deterministic acceptance)}.}
\label{alg:samp}
\paragraph{Notation.}
Let the unnormalized power target \(\tilde \pi_{\alpha}(y\mid x)\propto \pi(y\mid x)^{\alpha}\).
Let \(A(y',y)\) denote the Metropolis--Hastings acceptance ratio comparing completions \(y,y'\) (with \(x\) fixed), where \(p_{\mathrm{prop}}(y'\mid y,x)\) denotes the autoregressive proposal density for resampling a suffix under \(p_{\mathrm{prop}}(\cdot\mid x,\cdot)\):
\begin{equation}
\label{eq:accept}
A(y',y)
\;:=\;
\min\left\{1,\
\frac{\tilde \pi_{\alpha}(y'\mid x)}{\tilde \pi_{\alpha}(y\mid x)}
\cdot
\frac{p_{\mathrm{prop}}(y\mid y',x)}{p_{\mathrm{prop}}(y'\mid y,x)}
\right\}.
\end{equation}
\begin{algorithmic}[1]
\Require Base model \(\pi\); proposal \(p_{\mathrm{prop}}\); prompt \(x\); completion length \(T\) with \(B\mid T\); block size \(B\); inner iterations \(N_{\mathrm{MCMC}}\); exponent \(\alpha>0\).
\Ensure Completion \(y_{1:T}\) (\(\mathrm{MH}\): approximate sample from powered conditional \(\pi(\cdot\mid x)^{\alpha}\) up to MCMC error; Power$(\infty)$: accept proposals only if \(\pi(y'\mid x)>\pi(y\mid x)\), so \(\pi(y\mid x)\) is monotone along accepted moves).
\For{\(k=0,1,\dots,T/B-1\)}
    \State Given the current state \(y_{1:kB}\), construct an initialization \(y^{(0)}\) by extending autoregressively with \(p_{\mathrm{prop}}\) to length \((k+1)B\):
    \[
    y^{(0)}_t \sim p_{\mathrm{prop}}(y_t \mid x, y_{<t}),
    \qquad kB+1 \le t \le (k+1)B.
    \]
    \State Set \(y\gets y^{(0)}\).
    \For{\(n=1,\dots,N_{\mathrm{MCMC}}\)}
        \State Sample \(m\) uniformly from \(\{1,\dots,(k+1)B\}\).
        \State Construct a proposal completion \(y'\) with prefix \(y_{1:m-1}\) and resample the suffix:
        \[
        y'_{t} \sim p_{\mathrm{prop}}(y_t \mid x, y'_{<t}),
        \qquad m \le t \le (k+1)B.
        \]
        \State \textbf{(MH)} Compute \(A(y',y)\) from Eq.~\eqref{eq:accept}. Draw \(u\sim\mathrm{Uniform}(0,1)\). If \(u\le A(y',y)\), set \(y\gets y'\).
        \State \textcolor{algvariant}{\textbf{(Power$(\infty)$)} If \(\pi(y'\mid x) > \pi(y\mid x)\), set \(y\gets y'\).}
    \EndFor
    \State Set \(y_{1:(k+1)B}\gets y\) as the current state carried into the next block iteration.
\EndFor \\
\Return \(y_{1:T}\)
\end{algorithmic}
\end{algorithm}

\begin{table*}[t]
  \centering
  \setlength{\tabcolsep}{2.6pt}
  \caption{Comparison of prior work by its connection to power distributions, sampling, RL, distillation, and whether it avoids external rewards.}
  \label{tab:comparison}
  \begin{tabular*}{\linewidth}{@{\extracolsep{\fill}}>{\raggedright\arraybackslash}m{0.27\linewidth}ccccc@{}}
  \toprule
  Paper & Power distribution & Sampling & RL & Distillation & No external reward \\
  \midrule
  \citet{norouzi2016reward} & -- & -- & ✓ & ✓ & -- \\ 
  \citet{rusu2016policy} & -- & -- & ✓ & ✓ & -- \\ 
  \citet{teh2017distral} & -- & -- & ✓ & ✓ & -- \\ 
  \citet{laskin2023incontext} & -- & -- & ✓ & ✓ & -- \\ 
  \citet{huang2025selfimprovement} & -- & ✓ & ✓ & ✓ & ✓ \\ 
  \citet{gui2024bonbon} & -- & ✓ & -- & ✓ & -- \\ 
  \citet{amini2025variational} & -- & ✓ & -- & ✓ & -- \\ 
  \citet{balashankar2025infalign} & -- & ✓ & ✓ & -- & -- \\ 
  \citet{sessa2025bond} & -- & ✓ & ✓ & ✓ & -- \\ 
  \citet{yang2025fasterwind} & -- & ✓ & ✓ & ✓ & -- \\ 
  \citet{karan2026reasoning} & ✓ & ✓ & -- & -- & ✓ \\ 
  \citet{azizi2026power} & ✓ & ✓ & -- & -- & ✓ \\ 
  \citet{ji2026scalable} & ✓ & ✓ & -- & -- & ✓ \\ 
  \textbf{Ours} & ✓ & ✓ & ✓ & ✓ & ✓ \\
  \bottomrule
  \end{tabular*}
\end{table*}

\clearpage

\section{Proofs and background}

\subsection[Proof of the power--temperature odds-ratio proposition]{Proof of \Cref{prop:pow-vs-temp-renyi}}
\label{app:proof-pow-vs-temp-renyi}

\begin{proof}[Proof of \Cref{prop:pow-vs-temp-renyi}]
Fix \(x\) and a prefix \(y_{<t}\) with \(\pi(y_{<t}\mid x)>0\), and write \(p(s):=\pi(y_{t}=s\mid x,y_{<t})\) for \(s\in\mathcal{V}\).
For any suffix \(y_{t+1:T}\) and token \(s\), autoregressive factorization gives
\[
\pi(y_{<t},\,s,\,y_{t+1:T}\mid x)
=
\pi(y_{<t}\mid x)\,p(s)\,q_{t,s}(y_{t+1:T}).
\]
Using Eq.~\eqref{eq:pow-marginal-token}, the numerator for token \(s\) is therefore
\begin{align*}
\sum_{y_{t+1:T}\in\mathcal{V}^{T-t}}\pi(y_{<t},\,s,\,y_{t+1:T}\mid x)^{\alpha}
&=
\pi(y_{<t}\mid x)^{\alpha}\,p(s)^{\alpha}\sum_{y_{t+1:T}}q_{t,s}(y_{t+1:T})^{\alpha}.
\end{align*}
Summing over \(s\in\mathcal{V}\) yields the corresponding denominator in Eq.~\eqref{eq:pow-marginal-token}, so the prefix factor \(\pi(y_{<t}\mid x)^{\alpha}\) cancels and
\begin{equation}
\label{eq:pow-marg-simplify}
\pi_{\mathrm{pow},\alpha}(y_{t}=s \mid x, y_{<t})
=
\frac{p(s)^{\alpha}\sum_{z}q_{t,s}(z)^{\alpha}}{\sum_{s'\in\mathcal{V}}p(s')^{\alpha}\sum_{z}q_{t,s'}(z)^{\alpha}}.
\end{equation}
For temperature scaling, Eq.~\eqref{eq:temp-marginal-token} gives \(\pi_{\mathrm{temp},\alpha}(y_{t}=s\mid x,y_{<t})=p(s)^{\alpha}/\sum_{s'}p(s')^{\alpha}\).
Hence for any \(a,b\in\mathcal{V}\) with \(p(b)>0\),
\[
\frac{\pi_{\mathrm{pow},\alpha}(y_{t}=a \mid x, y_{<t})}{\pi_{\mathrm{pow},\alpha}(y_{t}=b \mid x, y_{<t})}
=
\frac{p(a)^{\alpha}\sum_{z}q_{t,a}(z)^{\alpha}}{p(b)^{\alpha}\sum_{z}q_{t,b}(z)^{\alpha}},
\qquad
\frac{\pi_{\mathrm{temp},\alpha}(y_{t}=a \mid x, y_{<t})}{\pi_{\mathrm{temp},\alpha}(y_{t}=b \mid x, y_{<t})}
=
\left(\frac{p(a)}{p(b)}\right)^{\alpha},
\]
and therefore
\[
\frac{\pi_{\mathrm{pow},\alpha}(y_{t}=a \mid x, y_{<t})}{\pi_{\mathrm{pow},\alpha}(y_{t}=b \mid x, y_{<t})}
\bigg/
\frac{\pi_{\mathrm{temp},\alpha}(y_{t}=a \mid x, y_{<t})}{\pi_{\mathrm{temp},\alpha}(y_{t}=b \mid x, y_{<t})}
=
\frac{\sum_{z}q_{t,a}(z)^{\alpha}}{\sum_{z}q_{t,b}(z)^{\alpha}}.
\]
By the definition of \(H_{\alpha}\), \(\sum_{z}q(z)^{\alpha}=\exp\!\bigl((1-\alpha)H_{\alpha}(q)\bigr)\), which yields Eq.~\eqref{eq:pow-temp-ratio-pair}.
\end{proof}

\subsection[Background and proof of the variance-minimizing proposal]{Background and proof of \Cref{prop:opt-proposal-power}}
\label{app:sis-background}

This appendix is aligned with the sequential Monte Carlo presentation of \citet{zhao2024probabilistic}, who derive a general \emph{twist-induced} proposal (their Prop.~3.3) that minimizes the variance of the one-step incremental importance weight for a given tower of intermediate targets.
We provide a proof of the same variance-minimization fact specialized to \(\pi_\alpha\) using the Cauchy--Schwarz inequality (cf.\ \citealt{zhao2024probabilistic}, App.~A.2).

\subsubsection{From a sequence-level target to a sequential sampler}
\label{app:sis-seq-target}
Let \(\mathcal{V}\) be a finite vocabulary and fix a prompt \(x\) and completion length \(T\ge 1\).
Let \(P(y_{1:T}):=\pi_\alpha(y_{1:T}\mid x)\) denote the power distribution on \(\mathcal{V}^T\) from Eq.~\eqref{eq:power-distribution}, i.e.,
\[
P(y_{1:T})\propto \pi(y_{1:T}\mid x)^{\alpha},
\qquad
\pi(y_{1:T}\mid x)=\prod_{t=1}^{T}\pi(y_t\mid x,y_{<t}).
\]
Exact sampling from \(P\) may be intractable because normalizing constants involve sums over exponentially many sequences.
Many practical samplers therefore build \(y_{1:T}\) \emph{sequentially}: having generated a prefix \(y_{<t}\), they draw a next token \(y_t\) from a proposal \(q(\cdot\mid x,y_{<t})\) and update importance weights so that, after \(T\) steps, full-length draws can be reweighted to be (exactly or approximately) correct for \(P\).

\subsubsection{Incremental importance weights}
\label{app:sis-inc-weights}
For \(t=1,\dots,T\), let \(P_t\) denote the marginal of \(P\) on the length-\(t\) prefix:
\[
P_t(y_{1:t}):=\sum_{y_{t+1:T}\in\mathcal{V}^{T-t}} P(y_{1:T}).
\]
One step of sequential importance sampling extends \(y_{<t}\) by sampling \(Y_t\sim q(\cdot\mid x,y_{<t})\).
The \emph{incremental} multiplicative factor appended to the running weight is~\citep{chopin2020introduction}
\begin{equation}
\label{eq:app-inc-weight}
W_t
:=
\frac{P_t(y_{<t},Y_t)}{P_{t-1}(y_{<t})\,q(Y_t\mid x,y_{<t})},
\end{equation}
defined on the event \(\{P_{t-1}(y_{<t})>0\}\), where \((y_{<t},Y_t)\) denotes the length-\(t\) prefix ending in \(Y_t\).
For the power distribution, \(P_t(y_{1:t})=\pi_\alpha(y_{1:t}\mid x)\), so Eq.~\eqref{eq:app-inc-weight} agrees with \(W_t\) in Eq.~\eqref{eq:inc-weight-def}.

If one initializes weights at \(w_0:=1\) and updates \(w_t:=w_{t-1}W_t\), then for any completed trajectory \(y_{1:T}\) with \(\prod_{t=1}^T P_{t-1}(y_{<t})>0\),
\begin{equation}
\label{eq:app-full-weight}
w_T
=
\frac{P(y_{1:T})}{q(y_{1:T}\mid x)},
\qquad
q(y_{1:T}\mid x):=\prod_{t=1}^{T} q(y_t\mid x,y_{<t}),
\end{equation}
which is the usual full-sequence importance weight of \(y_{1:T}\) for the target \(P\) against the autoregressive proposal \(q\).
Thus each \(W_t\) is the local factor that must be ``well behaved'' if the final weights are not to explode or collapse.

\subsubsection[Why minimize the one-step variance?]{Why minimize \(\mathrm{Var}[W_t]\) at one step?}
\label{app:sis-why-var}
Condition on a fixed feasible prefix \(y_{<t}\) with \(P_{t-1}(y_{<t})>0\).
Write \(f_t(v):=P_t(y_{<t},v)/P_{t-1}(y_{<t})\) for \(v\in\mathcal{V}\), i.e., the \emph{true} conditional \(P(y_t=v\mid y_{<t})\) under \(P\).
Then \(W_t=f_t(Y_t)/q(Y_t)\) with \(Y_t\sim q\).

Whenever \(q(v)>0\) for all \(v\) with \(f_t(v)>0\), the mean is always \(\mathbb{E}[W_t\mid y_{<t}]=\sum_{v\in\mathcal{V}} q(v)\,f_t(v)/q(v)=1\).
However, \(\mathrm{Var}[W_t\mid y_{<t}]\) depends strongly on \(q\): if \(q\) places too little mass where \(f_t\) is large, occasional huge weights arise, which is the usual ``weight degeneracy'' pathology in importance sampling.
Minimizing \(\mathrm{Var}[W_t\mid y_{<t}]\) therefore makes the \emph{single-step} contribution to weight instability as small as possible (among independent proposals), holding the prefix fixed.
This is the same local objective highlighted by \citet{zhao2024probabilistic} for twist-induced proposals.

\subsubsection[Proof of the variance-minimizing proposal]{Proof of \Cref{prop:opt-proposal-power}}
\label{app:proof-opt-proposal-power}

\begin{proof}[Proof of \Cref{prop:opt-proposal-power}]
Fix \(y_{<t}\) with \(\pi_\alpha(y_{<t}\mid x)>0\) and write \(f(v):=\pi_\alpha(y_{<t},v\mid x)/\pi_\alpha(y_{<t}\mid x)\) for \(v\in\mathcal{V}\).
Then \(\sum_{v\in\mathcal{V}} f(v)=1\) and \(W_t=f(Y_t)/q(Y_t)\) under \(Y_t\sim q\), assuming \(q(v)>0\) whenever \(f(v)>0\).

Since \(\mathbb{E}[W_t]=\sum_v f(v)=1\),
\[
\mathrm{Var}[W_t]
=
\mathbb{E}[W_t^2]-1
=
\sum_{v\in\mathcal{V}} \frac{f(v)^2}{q(v)}-1.
\]
By Cauchy--Schwarz,
\[
\Big(\sum_{v\in\mathcal{V}} f(v)\Big)^2
=
\Big(\sum_{v\in\mathcal{V}} \sqrt{q(v)}\cdot \frac{f(v)}{\sqrt{q(v)}}\Big)^2
\le
\Big(\sum_{v\in\mathcal{V}} q(v)\Big)\Big(\sum_{v\in\mathcal{V}} \frac{f(v)^2}{q(v)}\Big)
=
\sum_{v\in\mathcal{V}} \frac{f(v)^2}{q(v)}.
\]
The left-hand side equals \(1\), so \(\mathrm{Var}[W_t]\ge 0\) with equality if and only if the Cauchy--Schwarz inequality is tight, i.e., \(\sqrt{q(v)}\propto f(v)/\sqrt{q(v)}\), equivalently \(q(v)\propto f(v)\).
Because \(\sum_v f(v)=1\), the unique minimizer on \(\{v:f(v)>0\}\) is \(q(v)=f(v)\), which is \(\pi_\alpha(v\mid x,y_{<t})\).

Finally, with \(\tilde\pi_{\alpha,t}\) as in Eq.~\eqref{eq:power-prefix-marg} and \(Z_\alpha\) as in the main text,
\[
f(v)
=
\frac{\pi_\alpha(y_{<t},v\mid x)}{\pi_\alpha(y_{<t}\mid x)}
=
\frac{\tilde\pi_{\alpha,t}(y_{1:t})/Z_\alpha}{\tilde\pi_{\alpha,t-1}(y_{<t})/Z_\alpha}
=
\frac{\tilde\pi_{\alpha,t}(y_{1:t})}{\tilde\pi_{\alpha,t-1}(y_{<t})},
\]
where \(y_{1:t}=(y_{<t},v)\), which is Eq.~\eqref{eq:opt-q-power}.
\end{proof}

\subsection[Proof of power self-distillation and sharpening]{Proof of \Cref{prop:teacher-sft}}
\label{app:proof-teacher-sft}

We first provide the following lemma, which is used to bound the Hellinger distance between the MLE and the true conditional distribution for finite-class models.

\begin{lemma}[Finite-class MLE Hellinger bound~\citep{wong1995probability,geer2000empirical,zhang2006f}]
\label{lem:mle-hellinger}
Assume \(|\Pi|<\infty\) and \(\pi^{\star}\in\Pi\). Let \(D=\{(x_i,y_i)\}_{i=1}^{n}\) be i.i.d. with \(x_i\sim\mu\) and \(y_i\sim\pi^{\star}(\cdot\mid x_i)\),
and let \(\widehat\pi\in\arg\max_{\pi\in\Pi}\sum_{i=1}^{n}\log \pi(y_i\mid x_i)\) be an MLE.
Then for any \(\rho\in(0,1)\), with probability at least \(1-\rho\),
\[
\mathbb{E}_{x\sim\mu}\!\left[D_H^2\!\big(\widehat\pi(\cdot\mid x),\pi^{\star}(\cdot\mid x)\big)\right]
\le
\frac{2\log(|\Pi|\rho^{-1})}{n}.
\]
\end{lemma}

Using this lemma, we can prove~\Cref{prop:teacher-sft} as follows.

\begin{proof}[Proof of~\Cref{prop:teacher-sft}]
Define the failure event \(F(x):=\{\widehat\pi(\bm{y}^\star(x)\mid x)\le 1-\delta\}\).
By a simple inclusion,
\[
F(x)\subseteq
\Big\{\pi_\alpha(\bm{y}^\star(x)\mid x)\le 1-\tfrac{\delta}{2}\Big\}
\ \cup\
\Big\{\widehat\pi(\bm{y}^\star(x)\mid x)\le 1-\delta,\ \pi_\alpha(\bm{y}^\star(x)\mid x)>1-\tfrac{\delta}{2}\Big\}.
\]
Taking \(\mathbb{P}_{x\sim\mu}[\cdot]\) yields
\begin{equation}
\label{eq:teacher-sft-step1}
\mathbb{P}_{x\sim\mu}[F(x)]
\le
\mathbb{P}_{x\sim\mu}\big[\pi_\alpha(\bm{y}^\star(x)\mid x)\le 1-\tfrac{\delta}{2}\big]
\;+\;
\mathbb{P}_{x\sim\mu}[E(x)],
\end{equation}
where \(E(x):=\{\widehat\pi(\bm{y}^\star(x)\mid x)\le 1-\delta,\ \pi_\alpha(\bm{y}^\star(x)\mid x)>1-\tfrac{\delta}{2}\}\).

Let \(B(x):=\mathcal{Y}\setminus \bm{y}^\star(x)\).
For each \(x\), write \(\widehat\pi_x:=\widehat\pi(\cdot\mid x)\) and \(p_x:=\pi_\alpha(\cdot\mid x)\).
For two distributions \(p,q\in\Delta(\mathcal{Y})\), define the squared Hellinger distance
\[
D_H^2(p,q):=\sum_{y\in\mathcal{Y}}\big(\sqrt{p(y)}-\sqrt{q(y)}\big)^2.
\]
By the reverse triangle inequality applied to the vectors \((\sqrt{\widehat\pi_x(y)})_{y\in B(x)}\) and \((\sqrt{p_x(y)})_{y\in B(x)}\),
\begin{equation}
\label{eq:teacher-sft-hellinger-lb}
D_H^2(\widehat\pi_x,p_x)
\ge
\sum_{y\in B(x)}\big(\sqrt{\widehat\pi_x(y)}-\sqrt{p_x(y)}\big)^2
\ge
\big(\sqrt{\widehat\pi(B(x)\mid x)}-\sqrt{\pi_\alpha(B(x)\mid x)}\big)^2.
\end{equation}
On the event \(E(x)\), we have \(\widehat\pi(B(x)\mid x)=1-\widehat\pi(\bm{y}^\star(x)\mid x)\ge\delta\) and
\(\pi_\alpha(B(x)\mid x)=1-\pi_\alpha(\bm{y}^\star(x)\mid x)<\delta/2\), so Equation \eqref{eq:teacher-sft-hellinger-lb} implies
\[
D_H^2(\widehat\pi_x,p_x)\ge \big(\sqrt{\delta}-\sqrt{\delta/2}\big)^2
=\Big(1-\tfrac{1}{\sqrt2}\Big)^2\delta=:c_0\,\delta.
\]
Therefore \(\mathbf{1}\{E(x)\}\le D_H^2(\widehat\pi_x,p_x)/(c_0\delta)\), and hence
\begin{equation}
\label{eq:teacher-sft-step2}
\mathbb{P}_{x\sim\mu}[E(x)]
\le
\frac{1}{c_0\delta}\,\mathbb{E}_{x\sim\mu}\big[D_H^2(\widehat\pi_x,p_x)\big].
\end{equation}

Finally, by the finite-class MLE Hellinger bound (Lemma \ref{lem:mle-hellinger}) and \(|\Pi_\alpha|\le M\), with probability at least \(1-\rho\),
\[
\mathbb{E}_{x\sim\mu}\big[D_H^2(\widehat\pi_x,p_x)\big]
\le
\frac{2\log(M\rho^{-1})}{n}.
\]
Combining Eqs.~\eqref{eq:teacher-sft-step1} and \eqref{eq:teacher-sft-step2} with the bound above yields
\[
\mathbb{P}_{x\sim\mu}\big[\widehat\pi(\bm{y}^\star(x)\mid x)\le 1-\delta\big]
\le
\mathbb{P}_{x\sim\mu}\big[\pi_\alpha(\bm{y}^\star(x)\mid x)\le 1-\tfrac{\delta}{2}\big]
\;+\;
\frac{2}{c_0}\cdot\frac{\log(M\rho^{-1})}{\delta\,n}.
\]
Absorbing constants proves Eq.~\eqref{eq:teacher-distillation}.

\paragraph{Convergence of the upper bound.}
The MLE term satisfies
\(
\frac{\log(M\rho^{-1})}{\delta n}\to 0
\)
as \(n\to\infty\).

For the limit of \(\alpha\), fix \(x\in\mathcal{X}\) and write \(m(x):=\max_{y\in\mathcal{Y}}\pi(y\mid x)\).
By definition of \(\bm{y}^\star(x)\), we have \(\pi(y\mid x)=m(x)\) for all \(y\in\bm{y}^\star(x)\) and \(\pi(y\mid x)<m(x)\) for all \(y\in B(x)=\mathcal{Y}\setminus\bm{y}^\star(x)\).
The normalizing constant of the power distribution satisfies
\begin{align}
    Z_\alpha(x)&:=\sum_{y'\in\mathcal{Y}}\pi(y'\mid x)^{\alpha} \\
&=
\sum_{y'\in\bm{y}^\star(x)}m(x)^{\alpha}
\;+\;
\sum_{y'\in B(x)}\pi(y'\mid x)^{\alpha} \\
&=
\bigl|\bm{y}^\star(x)\bigr|\,m(x)^{\alpha}
\;+\;
\sum_{y'\in B(x)}\pi(y'\mid x)^{\alpha}.
\end{align}
For each \(y'\in B(x)\), the ratio \(\pi(y'\mid x)/m(x)\) lies in \([0,1)\), hence \(\bigl(\pi(y'\mid x)/m(x)\bigr)^{\alpha}\to 0\) as \(\alpha\to\infty\).
Because \(B(x)\) is finite, \(\sum_{y'\in B(x)}\pi(y'\mid x)^{\alpha}=o\bigl(m(x)^{\alpha}\bigr)\), and therefore
\begin{equation}
\label{eq:power-mass-to-mode}
\pi_\alpha(\bm{y}^\star(x)\mid x)
=
\frac{\bigl|\bm{y}^\star(x)\bigr|\,m(x)^{\alpha}}{Z_\alpha(x)}
\xrightarrow[\alpha\to\infty]{}1.
\end{equation}
The indicators \(\mathbf{1}\{\pi_\alpha(\bm{y}^\star(x)\mid x)\le 1-\tfrac{\delta}{2}\}\) converge to \(0\) for \(\mu\)-almost every \(x\) as \(\alpha\to\infty\) by Eq.~\eqref{eq:power-mass-to-mode}.
Since indicators are bounded by \(1\), dominated convergence yields
\[
\mathbb{P}_{x\sim\mu}\big[\pi_\alpha(\bm{y}^\star(x)\mid x)\le 1-\tfrac{\delta}{2}\big]
=
\mathbb{E}_{x\sim\mu}\big[\mathbf{1}\{\pi_\alpha(\bm{y}^\star(x)\mid x)\le 1-\tfrac{\delta}{2}\}\big]
\xrightarrow[\alpha\to\infty]{}0.
\]
Thus, the second term in Eq.~\eqref{eq:teacher-distillation} converges to \(0\) as \(\alpha\to\infty\).
Together with the \(n\to\infty\) limit of the first term, the full upper bound converges to \(0\).
\end{proof}

\subsection[Proof of the covariance derivative proposition]{Proof of \Cref{prop:Rprime-cov}}
\label{app:proof-Rprime-cov}

\begin{proof}[Proof of \Cref{prop:Rprime-cov}]
Recall
\[
\pi_\alpha(y)=\frac{\pi(y)^\alpha}{Z_\alpha}=\frac{e^{\alpha \log \pi(y)}}{Z_\alpha},
\qquad
Z_\alpha:=\sum_{y'\in\mathcal{Y}}\pi(y')^\alpha=\sum_{y'}e^{\alpha \log \pi(y')}.
\]

Differentiating \(\pi_\alpha(y)\) with respect to \(\alpha\) yields
\begin{align*}
\frac{\partial}{\partial\alpha}\pi_\alpha(y)
&=\frac{\partial}{\partial\alpha}\left(\frac{e^{\alpha \log \pi(y)}}{Z_\alpha}\right) \\
&=\frac{\log \pi(y)e^{\alpha \log \pi(y)}Z_\alpha-e^{\alpha \log \pi(y)} \frac{\partial}{\partial \alpha} Z_\alpha}{Z_\alpha^2}\\
&=\pi_\alpha(y)\left(\log \pi(y)-\frac{\partial}{\partial\alpha}\log Z_\alpha\right).
\end{align*}

Using
\begin{align}
\frac{\partial}{\partial\alpha}\log Z_\alpha
=\frac{1}{Z_\alpha}\sum_{y'} \log \pi(y')e^{\alpha \log \pi(y')}
=\sum_{y'}\pi_\alpha(y')\log \pi(y')
=\mathbb{E}_{\pi_\alpha}[\log\pi],
\end{align}
we obtain
\begin{align}
\frac{\partial}{\partial\alpha}\mathbb{E}_{\pi_\alpha}[r^\star]
&=\sum_y r^\star(y)\frac{\partial}{\partial\alpha}\pi_\alpha(y) \\
&=\sum_y r^\star(y)\pi_\alpha(y)\left(\log \pi(y)-\mathbb{E}_{\pi_\alpha}[\log \pi]\right)\\
&=\mathbb{E}_{\pi_\alpha}\!\big[r^\star(\log \pi-\mathbb{E}_{\pi_\alpha}[\log \pi])\big] \\
&=\mathrm{Cov}_{\pi_\alpha}(r^\star,\log \pi) \\
&=\mathrm{Cov}_{\pi_\alpha}(r^\star,r_{\mathrm{self}}).
\end{align}
\end{proof}

\subsection[Proof and extensions for power-teacher gain transfer]{Proof and extensions for \Cref{prop:gain-transfer}}
\label{app:proof-gain-transfer}

\begin{proof}[Proof of \Cref{prop:gain-transfer}]
Insert the power teacher between \(q\) and the base model:
\begin{equation}
\label{eq:app-gain-decomposition}
\mathcal{R}(q)-\mathcal{R}(\pi)
=
\bigl(\mathcal{R}(\pi_\alpha)-\mathcal{R}(\pi)\bigr)
+
\bigl(\mathcal{R}(q)-\mathcal{R}(\pi_\alpha)\bigr).
\end{equation}
By \Cref{prop:Rprime-cov} and the fundamental theorem of calculus,
\[
\mathcal{R}(\pi_\alpha)-\mathcal{R}(\pi)
=
\mathbb{E}_{x\sim\mu}
\left[
\int_1^\alpha
\mathrm{Cov}_{y\sim\pi_s(\cdot\mid x)}
\bigl(r^\star(x,y),r_{\mathrm{self}}(x,y)\bigr)\,\mathrm{d}s
\right]
=
\Gamma_\alpha.
\]
For the student-transfer term, define
\[
f_x(y):=\frac{r^\star(x,y)-r_{\min}}{\Delta_r}\in[0,1].
\]
The variational characterization of total variation gives
\[
\left|
\mathbb{E}_{y\sim q(\cdot\mid x)}[f_x(y)]
-
\mathbb{E}_{y\sim\pi_\alpha(\cdot\mid x)}[f_x(y)]
\right|
\le
\mathrm{TV}\bigl(q(\cdot\mid x),\pi_\alpha(\cdot\mid x)\bigr).
\]
Rescaling by \(\Delta_r\) and averaging over \(x\sim\mu\) yields
\[
\mathcal{R}(q)-\mathcal{R}(\pi_\alpha)
\ge
-
\Delta_r
\mathbb{E}_{x\sim\mu}
\!\left[
\mathrm{TV}\bigl(q(\cdot\mid x),\pi_\alpha(\cdot\mid x)\bigr)
\right].
\]
Combining this inequality with Eq.~\eqref{eq:app-gain-decomposition} proves Eq.~\eqref{eq:tv-transfer-main}.

For the forward-KL bound, Pinsker's inequality and Jensen's inequality give
\[
\begin{aligned}
\mathbb{E}_{x\sim\mu}
\mathrm{TV}\bigl(q(\cdot\mid x),\pi_\alpha(\cdot\mid x)\bigr)
&=
\mathbb{E}_{x\sim\mu}
\mathrm{TV}\bigl(\pi_\alpha(\cdot\mid x),q(\cdot\mid x)\bigr)\\
&\le
\mathbb{E}_{x\sim\mu}
\sqrt{\frac{
D_{\mathrm{KL}}\!\bigl(\pi_\alpha(\cdot\mid x)\|q(\cdot\mid x)\bigr)
}{2}}\\
&\le
\sqrt{
\frac{
\mathbb{E}_{x\sim\mu}
D_{\mathrm{KL}}\!\bigl(\pi_\alpha(\cdot\mid x)\|q(\cdot\mid x)\bigr)
}{2}
}.
\end{aligned}
\]
Substituting this into Eq.~\eqref{eq:tv-transfer-main} proves Eq.~\eqref{eq:forward-transfer-main}.
\end{proof}

\begin{corollary}[Reverse-KL and self-reward RL transfer]
\label{cor:reverse-gain-transfer}
Define
\[
\mathsf{D}_{\mathrm{rev}}(q,\pi_\alpha)
:=
\mathbb{E}_{x\sim\mu}
D_{\mathrm{KL}}\!\bigl(q(\cdot\mid x)\|\pi_\alpha(\cdot\mid x)\bigr).
\]
If \(\mathsf{D}_{\mathrm{rev}}(q,\pi_\alpha)<\infty\), then
\begin{equation}
\label{eq:reverse-transfer-main}
\mathcal{R}(q)-\mathcal{R}(\pi)
\ge
\Gamma_\alpha
-
\Delta_r
\sqrt{\frac{\mathsf{D}_{\mathrm{rev}}(q,\pi_\alpha)}{2}}.
\end{equation}
Moreover, let \(\beta>0\), \(\alpha=1+\beta^{-1}\), and define the self-reward RL optimality gap
\[
\varepsilon_{\mathrm{RL}}(q)
:=
J_\beta(\pi_\alpha;\pi,r_{\mathrm{self}})
-
J_\beta(q;\pi,r_{\mathrm{self}}).
\]
Whenever the objective values are finite,
\begin{equation}
\label{eq:rl-transfer-main}
\mathcal{R}(q)-\mathcal{R}(\pi)
\ge
\Gamma_\alpha
-
\Delta_r
\sqrt{\frac{\varepsilon_{\mathrm{RL}}(q)}{2\beta}}.
\end{equation}
Thus, if \(\Gamma_\alpha>0\), strict improvement follows whenever
\[
\mathsf{D}_{\mathrm{rev}}(q,\pi_\alpha)
<
\frac{2\Gamma_\alpha^2}{\Delta_r^2},
\qquad\text{or equivalently}\qquad
\varepsilon_{\mathrm{RL}}(q)
<
\frac{2\beta\Gamma_\alpha^2}{\Delta_r^2}.
\]
\end{corollary}

\begin{proof}
Pinsker's inequality in the reverse-KL direction and Jensen's inequality give
\[
\begin{aligned}
\mathbb{E}_{x\sim\mu}
\mathrm{TV}\bigl(q(\cdot\mid x),\pi_\alpha(\cdot\mid x)\bigr)
&\le
\mathbb{E}_{x\sim\mu}
\sqrt{
\frac{
D_{\mathrm{KL}}\!\bigl(q(\cdot\mid x)\|\pi_\alpha(\cdot\mid x)\bigr)
}{2}
}\\
&\le
\sqrt{\frac{\mathsf{D}_{\mathrm{rev}}(q,\pi_\alpha)}{2}}.
\end{aligned}
\]
Substitution into Eq.~\eqref{eq:tv-transfer-main} proves Eq.~\eqref{eq:reverse-transfer-main}.
By Eq.~\eqref{eq:rl-equals-kl}, the self-reward KL-regularized RL objective satisfies
\[
J_\beta(q;\pi,r_{\mathrm{self}})
=
-
\beta
\mathbb{E}_{x\sim\mu}
D_{\mathrm{KL}}\!\bigl(q(\cdot\mid x)\|\pi_\alpha(\cdot\mid x)\bigr)
+
\beta\mathbb{E}_{x\sim\mu}[\log Z_r(x)].
\]
Evaluating this identity at \(q=\pi_\alpha\) and subtracting yields
\(\varepsilon_{\mathrm{RL}}(q)=\beta\mathsf{D}_{\mathrm{rev}}(q,\pi_\alpha)\), which proves Eq.~\eqref{eq:rl-transfer-main}.
\end{proof}

\begin{corollary}[Finite-sample MLE transfer]
\label{cor:finite-sample-gain-transfer}
Fix \(\rho\in(0,1)\). Suppose that \(\Pi_\alpha\) is finite, \(\pi_\alpha\in\Pi_\alpha\), and \(D=\{(x_i,y_i)\}_{i=1}^{n}\) consists of i.i.d.\ samples with \(x_i\sim\mu\) and \(y_i\sim\pi_\alpha(\cdot\mid x_i)\).
Let
\[
\widehat\pi
\in
\arg\max_{q\in\Pi_\alpha}
\sum_{i=1}^{n}\log q(y_i\mid x_i)
\]
be the MLE in Eq.~\eqref{eq:power-distill-mle}.
Then, with probability at least \(1-\rho\) over \(D\),
\begin{equation}
\label{eq:finite-sample-transfer-main}
\mathcal{R}(\widehat\pi)-\mathcal{R}(\pi)
\ge
\Gamma_\alpha
-
\Delta_r
\sqrt{
\frac{2\log\!\bigl(|\Pi_\alpha|\rho^{-1}\bigr)}{n}
}.
\end{equation}
In particular, the MLE improves on \(\pi\) whenever the finite-sample penalty in Eq.~\eqref{eq:finite-sample-transfer-main} is strictly smaller than \(\Gamma_\alpha\).
Equivalently, if \(\Gamma_\alpha>0\) and
\[
n
>
\frac{
2\Delta_r^2\log\!\bigl(|\Pi_\alpha|\rho^{-1}\bigr)
}{
\Gamma_\alpha^2
},
\]
then \(\mathcal{R}(\widehat\pi)>\mathcal{R}(\pi)\) with probability at least \(1-\rho\).
\end{corollary}

\begin{proof}
By \Cref{lem:mle-hellinger}, with probability at least \(1-\rho\),
\[
\mathbb{E}_{x\sim\mu}
D_H^2\!\bigl(
\widehat\pi(\cdot\mid x),
\pi_\alpha(\cdot\mid x)
\bigr)
\le
\frac{2\log\!\bigl(|\Pi_\alpha|\rho^{-1}\bigr)}{n}.
\]
With the definition
\[
D_H^2(p,q)
:=
\sum_{y\in\mathcal{Y}}
\bigl(\sqrt{p(y)}-\sqrt{q(y)}\bigr)^2,
\]
we have \(\mathrm{TV}(p,q)\le D_H(p,q)\).
Therefore Jensen's inequality gives
\[
\mathbb{E}_{x\sim\mu}
\mathrm{TV}\bigl(
\widehat\pi(\cdot\mid x),
\pi_\alpha(\cdot\mid x)
\bigr)
\le
\sqrt{
\frac{2\log\!\bigl(|\Pi_\alpha|\rho^{-1}\bigr)}{n}
}.
\]
Substituting this bound into Eq.~\eqref{eq:tv-transfer-main} proves Eq.~\eqref{eq:finite-sample-transfer-main}.
\end{proof}

\subsection{Closed-form optimizer for KL-regularized RL: restatement and proof}
\label{app:proof-kl-regularized-tilt}

We restate the standard closed-form solution of KL-regularized RL used in \Cref{sec:power-rl}.

\begin{proposition}[Closed-form optimizer for KL-regularized RL~\citep{levine2018reinforcement}]
\label{prop:kl-regularized-tilt}
For each \(x\in\mathcal{X}\), let \(\pi_\beta^\star\) be the reward-tilted distribution defined in Eq.~\eqref{eq:tilted-policy}, and assume \(Z_r(x)<\infty\) for every \(x\in\mathcal{X}\).
Then \(\pi_\beta^\star\) maximizes \(J_\beta(q;\pi,r)\) in Eq.~\eqref{eq:soft-rl-obj} over all \(q:\mathcal{X}\to\Delta(\mathcal{Y})\).
\end{proposition}

\begin{proof}[Proof of \Cref{prop:kl-regularized-tilt}]
Fix \(x\in\mathcal{X}\) and write \(f(y):=\pi(y\mid x)\exp\!\big(\beta^{-1}r(x,y)\big)\) and \(Z:=Z_r(x)=\sum_{y'\in\mathcal{Y}}f(y')\).
For any \(q(\cdot\mid x)\in\Delta(\mathcal{Y})\), expanding the KL divergence against \(\pi_\beta^\star(\cdot\mid x)\) gives
\begin{align*}
\mathbb{E}_{y\sim q(\cdot\mid x)}[r(x,y)]
-\beta\,D_{\mathrm{KL}}\!\big(q(\cdot\mid x)\,\|\,\pi(\cdot\mid x)\big)
&=
-\beta\sum_{y\in\mathcal{Y}} q(y\mid x)\log\frac{q(y\mid x)}{f(y)/Z}
\\
&=
-\beta\,D_{\mathrm{KL}}\!\big(q(\cdot\mid x)\,\|\,\pi_\beta^\star(\cdot\mid x)\big)
+\beta\log Z,
\end{align*}
where we used \(\pi_\beta^\star(y\mid x)=f(y)/Z\) from Eq.~\eqref{eq:tilted-policy}.
Since \(D_{\mathrm{KL}}(\cdot\,\|\,\pi_\beta^\star(\cdot\mid x))\ge 0\) with equality if and only if \(q(\cdot\mid x)=\pi_\beta^\star(\cdot\mid x)\), the inner objective is uniquely maximized at \(q(\cdot\mid x)=\pi_\beta^\star(\cdot\mid x)\).
Because \(J_\beta(q;\pi,r)\) is an expectation over \(x\sim\mu\) of these decoupled per-\(x\) objectives, the unique global maximizer is \(q=\pi_\beta^\star\).
\end{proof}

\clearpage
\section{Experimental details}
\label{sec_app:experimental-details}
\subsection{Setup}
\label{sec_app:experimental-setups}

\paragraph{Models and datasets.}
We used Qwen2.5-Math-7B~\citep{yang2024qwen2}, Qwen2.5-7B~\citep{yang2024qwen}, and Phi-3.5-mini-instruct~\citep{abdin2024phi} models on the following datasets.
\begin{itemize}
    \item \textbf{Mathematics.} We used the MATH dataset~\citep{lightman2024lets}, which consists of 12,500 competition-style math problems spanning seven categories (e.g., geometry, number theory, and precalculus), with 7,500 training and 5,000 test problems. For evaluation, we used MATH500, a randomly selected subset of the MATH test set standardized by OpenAI\footnote{\url{https://huggingface.co/datasets/HuggingFaceH4/MATH-500}}. For distillation, we sampled 500 examples from MATH with MATH500 removed\footnote{\url{https://raw.githubusercontent.com/rasbt/math_full_minus_math500/main/math_full_minus_math500.json}}.
    \item \textbf{Programming.} For evaluation, we used HumanEval~\citep{chen2021evaluating}, a set of $164$ handwritten programming problems covering algorithms, reasoning, mathematics, and language understanding; each problem includes unit tests, and a solution was correct if it passed all tests. For distillation, we used MBPP~\citep{austin2021program}, a benchmark of crowd-sourced Python programming problems designed to be solvable by entry-level programmers. We used $420$ questions from the sanitized subset, excluding the prompt split.
    \item \textbf{Multiple-choice science.} We used GPQA~\citep{rein2024gpqa}, a multiple-choice science benchmark (physics, chemistry, and biology) requiring advanced reasoning. For evaluation, we used GPQA-Diamond, a high-quality subset of $198$ questions. For distillation, we used the remaining $250$ GPQA questions after removing any overlap with GPQA-Diamond.
\end{itemize}

\paragraph{Standard sampling.}
We used HuggingFace \texttt{generate} with \texttt{do\_sample=True}. For the Qwen and Phi models, temperature, top-$p$, and top-$k$ use the Transformers defaults: temperature $1.0$, top-$p=1.0$, and top-$k=50$.

\paragraph{Power sampling.}
We used the power sampling algorithm of~\citet{karan2026reasoning}, largely following their hyperparameters. Specifically, we used \(\alpha=4.0\), maximum sampling token length $3072$, block size $192$, $N_{\mathrm{MCMC}}=10$, and the proposal LLM $p_{\mathrm{prop}}$ set to the base model with sampling temperature \(\tau=1/\alpha=0.25\). The token-wise Temperature baseline uses the same \(\tau\), applying the corresponding local power transform independently at each decoding step. For the randomly initialized model (RandW; \Cref{sec:numerical}), we instead used maximum token length $1024$ and $N_{\mathrm{MCMC}}=2$, because under the default settings (maximum token length $3072$ and $N_{\mathrm{MCMC}}=10$) EOS tokens rarely appeared for RandW and wall-clock sampling time became significantly longer.

\paragraph{Self-reward computation.}
To report \(r_{\mathrm{self}}\), we computed, under the corresponding base model, the average log-likelihood over completion tokens, excluding prompt tokens.
Our theoretical analysis assumes completions of a fixed length~\(T\), but in our experiments completion lengths vary across prompts and sampling methods, so we normalize by the number of completion tokens to remove length bias in \(r_{\mathrm{self}}\).

\paragraph{Synthetic random rewards.}
For the synthetic-reward probe in \Cref{fig:reward_corr}, each completion~\(y\) is mapped to a scalar in \([0,1)\) by applying SHA-256 to the UTF-8 encoding of~\(y\) and interpreting the leading 64 bits of the digest as an unsigned fraction. Let \(z_{\mathrm{self}}(y)\) and \(z_{r}(y)\) denote the z-scores of the self-reward \(r_{\mathrm{self}}(x,y)\) and of the hash reward above, each computed with the corresponding pooled global sample mean and sample standard deviation.
We then define
\[
r_{\lambda}(y) \;:=\; \lambda\, z_{\mathrm{self}}(y) + \sqrt{1-\lambda^{2}}\,\bigl(z_{r}(y)+\varepsilon(y)\bigr),
\qquad \lambda\in[-1,1],
\]
where the \(\varepsilon(y)\) are i.i.d.\ \(\mathcal{N}(0,\sigma^{2})\) with \(\sigma=0.5\).
\Cref{fig:reward_corr} sweeps~\(\lambda\) and plots the mean increase in~\(r_{\lambda}\) under power versus standard sampling against the empirical covariance between \(r_{\mathrm{self}}\) and \(r_{\lambda}\), using completions produced under standard sampling.
The construction is designed to sweep \(\mathrm{Cov}(r_{\lambda},r_{\mathrm{self}})\) in a controlled way; we plot empirical gain against this controlled covariance to visualize the qualitative rate prediction of \Cref{prop:Rprime-cov}.

\paragraph{Distillation.}
We trained the student with supervised fine-tuning on the offline power-sampled dataset. Concretely, we minimized the standard token-level cross-entropy loss of a causal language model on the teacher-generated completion, masking the prompt tokens (i.e., the loss was computed only on the completion tokens). The student was initialized from the base model and was trained with LoRA adapters (\(r{=}16\), \(\alpha{=}32\), dropout \(0.05\)) applied to \texttt{q\_proj,k\_proj,v\_proj,o\_proj,gate\_proj,up\_proj,down\_proj}. We trained the models for 3 epochs using the AdamW optimizer with a weight decay of 0.01 and a linear warmup ratio of 0.03. The learning rate was tuned per dataset and model as summarized in \Cref{tab:lr_distill}. We used per-device batch size 1 with 8 gradient accumulation steps, and enabled gradient checkpointing. We set the maximum sequence length to 1024 tokens to keep activation memory manageable on a single GPU. Teacher completions exceeding this cap were truncated, and the cross-entropy loss was computed on all in-window completion tokens. The truncation affected only a minority of completions (e.g., 83.6\% of Qwen2.5-Math-7B completions on MATH fit fully within the cap), and each in-window token still provides a valid distillation signal toward \(\pi_\alpha\).

\begin{table}[htbp]
\centering
\caption{Learning rate used for SFT distillation, per (dataset, model) pair.}
\label{tab:lr_distill}
\begin{tabular}{lccc}
\toprule
Dataset & Qwen2.5-7B & Qwen2.5-Math-7B & Phi-3.5-mini-instruct \\
\midrule
MATH           & $1 \times 10^{-5}$ & $1 \times 10^{-5}$ & $1 \times 10^{-3}$ \\
HumanEval/MBPP & $1 \times 10^{-5}$ & $1 \times 10^{-5}$ & $5 \times 10^{-4}$ \\
GPQA           & $1 \times 10^{-5}$ & $1 \times 10^{-4}$ & $2 \times 10^{-4}$ \\
\bottomrule
\end{tabular}
\end{table}

\paragraph{Hardware and execution time.}
All experiments were conducted on GPU nodes equipped with two Intel Xeon Platinum 8360Y CPUs, 512 GiB of host memory, and eight NVIDIA A100 GPUs with 40 GiB of memory each. On a single GPU, supervised fine-tuning of one student per dataset and model finished in under one hour, while teacher generation via power sampling (\Cref{alg:samp}) took more than one day per dataset and model.
The total compute is on the order of a few hundred A100-GPU-hours.


\clearpage

\section{Additional analyses}
\label{sec_app:additional-analyses}

\subsection{A forward--reverse KL gap for power teachers.}

Power self-distillation replaces the reverse-KL form induced by self-reward KL-regularized RL,
\[
D_{\mathrm{KL}}\!\big(q(\cdot\mid x)\,\|\,\pi_{\alpha}(\cdot\mid x)\big),
\]
with the offline forward-KL surrogate
\[
D_{\mathrm{KL}}\!\big(\pi_{\alpha}(\cdot\mid x)\,\|\,q(\cdot\mid x)\big).
\]
We now show that, for a power teacher, the effect of such a KL reversal is controlled by the variance of the same self-reward that defines the power distribution.
Fix a prompt \(x\) and omit \(x\) from the notation.
Let \(\pi,\rho\in\Delta(\mathcal{Y})\) be strictly positive distributions on a finite support, and define logits \(s(y):=\log\pi(y)\) and \(r(y):=\log\rho(y)\).
For a convex function \(F\), its convex conjugate is
\[
F^\star(u):=\sup_{p\in\Delta(\mathcal{Y})}\{\langle u,p\rangle-F(p)\}.
\]
When \(F(p)=\sum_y p(y)\log p(y)\) is the negative entropy on the simplex, this conjugate is the log-sum-exp function
\[
F^\star(u)=\log\sum_y \exp(u(y)).
\]
Thus \(\nabla F^\star(u)=\mathrm{softmax}(u)\).
Let \(D_{F^\star}\) denote the Bregman divergence induced by \(F^\star\).
Then
\[
D_{F^\star}(s\|r)=D_{\mathrm{KL}}(\rho\|\pi),
\qquad
D_{F^\star}(r\|s)=D_{\mathrm{KL}}(\pi\|\rho).
\]
Moreover, for any vector \(v\),
\[
v^\top \nabla^2 F^\star(u)v
=
\mathrm{Var}_{Y\sim\mathrm{softmax}(u)}[v(Y)].
\]
Applying a second-order Taylor expansion around the midpoint between \(r\) and \(s\) yields, for some intermediate logit distributions \(m_+,m_-\) along the logit-interpolation path \(\mathcal{I}(\pi,\rho)\),
\begin{align}
\label{eq:generic-forward-reverse-kl-gap-exact}
D_{\mathrm{KL}}(\rho\|\pi)
&=
D_{\mathrm{KL}}(\pi\|\rho)
+
\frac14
\left\{
\mathrm{Var}_{Y\sim m_+}
\!\left[\log\frac{\rho(Y)}{\pi(Y)}\right]
-
\mathrm{Var}_{Y\sim m_-}
\!\left[\log\frac{\rho(Y)}{\pi(Y)}\right]
\right\}.
\end{align}
Therefore,
\begin{equation}
\label{eq:generic-forward-reverse-kl-gap}
\left|
D_{\mathrm{KL}}(\rho\|\pi)
-
D_{\mathrm{KL}}(\pi\|\rho)
\right|
\le
\frac14
\sup_{m\in\mathcal{I}(\pi,\rho)}
\mathrm{Var}_{Y\sim m}
\!\left[\log\frac{\rho(Y)}{\pi(Y)}\right],
\end{equation}
where \(\mathcal{I}(\pi,\rho)\) denotes distributions of the form
\[
m_\lambda(y)\propto \rho(y)^{1-\lambda}\pi(y)^\lambda.
\]
Thus, for a generic teacher, the forward--reverse KL gap is controlled by the teacher--base log-density ratio, which can be large or ill-defined under support mismatch.

For the power teacher \(\rho=\pi_\alpha\), however,
\[
\log\frac{\pi_\alpha(Y)}{\pi(Y)}
=
(\alpha-1)\log\pi(Y)-\log Z_\alpha,
\qquad
Z_\alpha:=\sum_{y'}\pi(y')^\alpha.
\]
The constant \(-\log Z_\alpha\) does not affect variance, and the logit-interpolation path remains on the power path.
Consequently, for some \(\lambda_+\in[(\alpha+1)/2,\alpha]\) and \(\lambda_-\in[1,(\alpha+1)/2]\),
\begin{align}
\label{eq:power-forward-reverse-kl-gap-exact}
D_{\mathrm{KL}}(\pi_\alpha\|\pi)
&=
D_{\mathrm{KL}}(\pi\|\pi_\alpha)
+
\frac{(\alpha-1)^2}{4}
\left\{
\mathrm{Var}_{Y\sim\pi_{\lambda_+}}[\log\pi(Y)]
-
\mathrm{Var}_{Y\sim\pi_{\lambda_-}}[\log\pi(Y)]
\right\}.
\end{align}
Hence,
\begin{equation}
\label{eq:power-forward-reverse-kl-gap-bound}
\left|
D_{\mathrm{KL}}(\pi_\alpha\|\pi)
-
D_{\mathrm{KL}}(\pi\|\pi_\alpha)
\right|
\le
\frac{(\alpha-1)^2}{4}
\sup_{\lambda\in[1,\alpha]}
\mathrm{Var}_{Y\sim\pi_\lambda}[\log\pi(Y)].
\end{equation}
It shows that, unlike in generic distillation with an arbitrary teacher, the absolute gap induced by reversing the KL in power self-distillation is not an arbitrary teacher-mismatch quantity.

\clearpage

\section{Additional experimental results}

\subsection{Other datasets and models}
\label{sec_app:other}
This section reports results on additional dataset--model combinations that are not shown in the main text.
In all cases, the distilled model has a higher~$r^\star$ than the base under standard autoregressive decoding.
The distilled model often attains~$r^\star$ comparable to that of the corresponding base model with power sampling.


\begin{table}[htbp]
\centering
\caption{MATH500: true reward~$r^{\star}$ (accuracy) and self-reward~$r_{\mathrm{self}}$. Left: all completions, means with $\pm$ std over seeds. Right: self-reward Best-of-\(N\) over samples generated with different seeds (max $r_{\mathrm{self}}$ per item, then same aggregation). 4 seeds. }
\label{tab:eval2-math}
\resizebox{\linewidth}{!}{
\begin{tabular}{llcccc}
\toprule
& & \multicolumn{2}{c}{All completions} & \multicolumn{2}{c}{Self-reward Best-of-\(N\)} \\
\cmidrule(lr){3-4} \cmidrule(lr){5-6}
Model & Sampling & $r^{\star} (\uparrow)$ & $r_{\mathrm{self}}$ & $r^{\star} (\uparrow)$ & $r_{\mathrm{self}}$ \\
\midrule
\multirow{2}{*}{Qwen / Base} & Standard & $0.410 \pm 0.004$ & $-0.405 \pm 0.049$ & $0.578$ & $-0.185$ \\
 & Power & $\mathbf{0.706} \pm 0.017$ & $-0.093 \pm 0.001$ & $0.676$ & $-0.080$ \\
\multirow{2}{*}{Qwen / Power-distilled} & Standard & $0.631 \pm 0.013$ & $-0.094 \pm 0.002$ & $\mathbf{0.682}$ & $-0.064$ \\
 & Temperature & $0.661 \pm 0.006$ & $\mathbf{-0.073} \pm 0.001$ & $0.676$ & $\mathbf{-0.060}$ \\
\midrule
\multirow{2}{*}{Phi / Base} & Standard & $0.449 \pm 0.014$ & $-0.234 \pm 0.001$ & $0.476$ & $-0.173$ \\
 & Power & $\mathbf{0.513} \pm 0.017$ & $-0.175 \pm 0.002$ & $\mathbf{0.492}$ & $-0.155$ \\
\multirow{2}{*}{Phi / Power-distilled} & Standard & $0.470 \pm 0.000$ & $-0.118 \pm 0.001$ & $0.480$ & $-0.088$ \\
 & Temperature & $0.457 \pm 0.014$ & $\mathbf{-0.106} \pm 0.001$ & $0.460$ & $\mathbf{-0.086}$ \\
\bottomrule
\end{tabular}}
\end{table}


\begin{table}[htbp]
\centering
\caption{HumanEval: true reward~$r^{\star}$ (HumanEval pass) and self-reward~$r_{\mathrm{self}}$. Left: all completions, means with $\pm$ std over seeds. Right: self-reward Best-of-\(N\) over samples generated with different seeds (max $r_{\mathrm{self}}$ per item, then same aggregation). 4 seeds. }
\label{tab:eval2-he}
\resizebox{\linewidth}{!}{
\begin{tabular}{llcccc}
\toprule
& & \multicolumn{2}{c}{All completions} & \multicolumn{2}{c}{Self-reward Best-of-\(N\)} \\
\cmidrule(lr){3-4} \cmidrule(lr){5-6}
Model & Sampling & $r^{\star} (\uparrow)$ & $r_{\mathrm{self}}$ & $r^{\star} (\uparrow)$ & $r_{\mathrm{self}}$ \\
\midrule
\multirow{2}{*}{Qwen-Math / Base} & Standard & $0.320 \pm 0.016$ & $-0.741 \pm 0.012$ & $0.384$ & $-0.427$ \\
 & Power & $0.538 \pm 0.030$ & $\mathbf{-0.144} \pm 0.003$ & $0.561$ & $\mathbf{-0.106}$ \\
\multirow{2}{*}{Qwen-Math / Power-distilled} & Standard & $0.416 \pm 0.023$ & $-0.563 \pm 0.040$ & $0.451$ & $-0.334$ \\
 & Temperature & $\mathbf{0.541} \pm 0.005$ & $-0.304 \pm 0.001$ & $\mathbf{0.567}$ & $-0.208$ \\
\midrule
\multirow{2}{*}{Qwen / Base} & Standard & $0.326 \pm 0.025$ & $-0.966 \pm 0.046$ & $0.378$ & $-0.426$ \\
 & Power & $\mathbf{0.573} \pm 0.020$ & $\mathbf{-0.130} \pm 0.004$ & $0.567$ & $\mathbf{-0.096}$ \\
\multirow{2}{*}{Qwen / Power-distilled} & Standard & $0.425 \pm 0.029$ & $-0.849 \pm 0.017$ & $0.470$ & $-0.325$ \\
 & Temperature & $0.541 \pm 0.017$ & $-0.479 \pm 0.024$ & $\mathbf{0.598}$ & $-0.235$ \\
\midrule
\multirow{2}{*}{Phi / Base} & Standard & $0.549 \pm 0.021$ & $-0.913 \pm 0.012$ & $0.561$ & $-0.589$ \\
 & Power & $0.712 \pm 0.027$ & $\mathbf{-0.330} \pm 0.004$ & $\mathbf{0.732}$ & $\mathbf{-0.294}$ \\
\multirow{2}{*}{Phi / Power-distilled} & Standard & $0.634 \pm 0.031$ & $-0.730 \pm 0.029$ & $0.604$ & $-0.473$ \\
 & Temperature & $\mathbf{0.715} \pm 0.020$ & $-0.627 \pm 0.028$ & $0.677$ & $-0.447$ \\
\bottomrule
\end{tabular}}
\end{table}


\begin{table}[htbp]
\centering
\caption{GPQA: true reward~$r^{\star}$ (accuracy) and self-reward~$r_{\mathrm{self}}$. Left: all completions, means with $\pm$ std over seeds. Right: self-reward Best-of-\(N\) over samples generated with different seeds (max $r_{\mathrm{self}}$ per item, then same aggregation). 4 seeds. }
\label{tab:eval2-gpqa}
\resizebox{\linewidth}{!}{
\begin{tabular}{llcccc}
\toprule
& & \multicolumn{2}{c}{All completions} & \multicolumn{2}{c}{Self-reward Best-of-\(N\)} \\
\cmidrule(lr){3-4} \cmidrule(lr){5-6}
Model & Sampling & $r^{\star} (\uparrow)$ & $r_{\mathrm{self}}$ & $r^{\star} (\uparrow)$ & $r_{\mathrm{self}}$ \\
\midrule
\multirow{2}{*}{Qwen-Math / Base} & Standard & $0.100 \pm 0.025$ & $-0.675 \pm 0.076$ & $0.101$ & $-0.675$ \\
 & Power & $\mathbf{0.277} \pm 0.022$ & $\mathbf{-0.088} \pm 0.002$ & $0.278$ & $\mathbf{-0.087}$ \\
\multirow{2}{*}{Qwen-Math / Power-distilled} & Standard & $0.275 \pm 0.004$ & $-0.165 \pm 0.001$ & $\mathbf{0.283}$ & $-0.113$ \\
 & Temperature & $\mathbf{0.277} \pm 0.001$ & $-0.149 \pm 0.001$ & $0.278$ & $-0.109$ \\
\midrule
\multirow{2}{*}{Qwen / Base} & Standard & $0.244 \pm 0.017$ & $-1.531 \pm 0.185$ & $0.247$ & $-1.527$ \\
 & Power & $0.283 \pm 0.033$ & $\mathbf{-0.118} \pm 0.001$ & $0.288$ & $\mathbf{-0.118}$ \\
\multirow{2}{*}{Qwen / Power-distilled} & Standard & $0.280 \pm 0.035$ & $-0.437 \pm 0.098$ & $0.278$ & $-0.426$ \\
 & Temperature & $\mathbf{0.285} \pm 0.025$ & $-0.210 \pm 0.007$ & $\mathbf{0.293}$ & $-0.201$ \\
\midrule
\multirow{2}{*}{Phi / Base} & Standard & $0.223 \pm 0.027$ & $-0.802 \pm 0.023$ & $0.222$ & $-0.800$ \\
 & Power & $\mathbf{0.309} \pm 0.019$ & $\mathbf{-0.215} \pm 0.004$ & $\mathbf{0.308}$ & $\mathbf{-0.214}$ \\
\multirow{2}{*}{Phi / Power-distilled} & Standard & $0.268 \pm 0.004$ & $-0.321 \pm 0.010$ & $0.268$ & $-0.319$ \\
 & Temperature & $0.292 \pm 0.012$ & $-0.284 \pm 0.059$ & $0.298$ & $-0.262$ \\
\bottomrule
\end{tabular}}
\end{table}

\clearpage

\paragraph{Llama-3.1-Tulu-3-8B-DPO.}
To further broaden model coverage, we evaluated
Llama-3.1-Tulu-3-8B-DPO~\citep{lambert2024tulu} on MATH500.
As shown in \Cref{tab:tulu-math}, power sampling and power
self-distillation improve both true reward and self-reward over the
base model, consistent with the trends observed for the other models.

\begin{table}[htbp]
\centering
\caption{MATH500 results for Llama-3.1-Tulu-3-8B-DPO. We report true reward~$r^{\star}$ (accuracy) and self-reward~$r_{\mathrm{self}}$.}
\label{tab:tulu-math}
\resizebox{\linewidth}{!}{
\begin{tabular}{llcccc}
\toprule
& & \multicolumn{2}{c}{All completions} & \multicolumn{2}{c}{Self-reward Best-of-$N$} \\
\cmidrule(lr){3-4} \cmidrule(lr){5-6}
Model & Sampling & $r^{\star} (\uparrow)$ & $r_{\mathrm{self}}$ & $r^{\star} (\uparrow)$ & $r_{\mathrm{self}}$ \\
\midrule
\multirow{3}{*}{Tulu / Base}
& Standard & $0.470$ & $-0.132$ & $0.474$ & $-0.104$ \\
& Temperature & $0.472$ & $-0.110$ & $0.508$ & $-0.089$ \\
& Power & $\mathbf{0.510}$ & $\mathbf{-0.104}$ & $\mathbf{0.510}$ & $-0.090$ \\
\midrule
\multirow{2}{*}{Tulu / Power-distilled}
& Standard & $0.494$ & $-0.112$ & $\mathbf{0.500}$ & $-0.103$ \\
& Temperature & $\mathbf{0.500}$ & $\mathbf{-0.100}$ & $0.492$ & $\mathbf{-0.093}$ \\
\bottomrule
\end{tabular}}
\end{table}

\paragraph{Harder mathematics benchmarks.}
We additionally evaluated
Qwen2.5-Math-7B and Phi-3.5-mini-instruct on AMC23 and AIME24.
As shown in \Cref{tab:amc-aime}, the results exhibit the same qualitative trend
on these harder mathematics benchmarks: power sampling and power
self-distillation consistently increase self-reward and often improve true
reward over standard sampling.

\begin{table}[htbp]
\centering
\caption{Results on AMC23 and AIME24 for Qwen2.5-Math-7B and Phi-3.5-mini-instruct. We report true reward~$r^{\star}$ (accuracy) and self-reward~$r_{\mathrm{self}}$.}
\label{tab:amc-aime}
\resizebox{\linewidth}{!}{
\begin{tabular}{lllcccc}
\toprule
& & & \multicolumn{2}{c}{All completions} & \multicolumn{2}{c}{Self-reward Best-of-$N$} \\
\cmidrule(lr){4-5} \cmidrule(lr){6-7}
Model & Dataset & Method & $r^{\star} (\uparrow)$ & $r_{\mathrm{self}}$ & $r^{\star} (\uparrow)$ & $r_{\mathrm{self}}$ \\
\midrule
\multirow{10}{*}{Qwen2.5-Math-7B}
& \multirow{5}{*}{AMC23}
& Base / Standard & $0.350$ & $-0.173$ & $0.400$ & $-0.087$ \\
& & Base / Temperature & $0.502$ & $-0.063$ & $0.575$ & $-0.045$ \\
& & Power sampling & $0.525$ & $-0.051$ & $0.525$ & $-0.051$ \\
& & Power-distilled / Standard & $0.506$ & $-0.063$ & $\mathbf{0.600}$ & $-0.031$ \\
& & Power-distilled / Temperature & $\mathbf{0.575}$ & $\mathbf{-0.030}$ & $\mathbf{0.600}$ & $\mathbf{-0.018}$ \\
\cmidrule(lr){2-7}
& \multirow{5}{*}{AIME24}
& Base / Standard & $0.075$ & $-0.225$ & $0.133$ & $-0.100$ \\
& & Base / Temperature & $0.083$ & $-0.068$ & $0.100$ & $-0.044$ \\
& & Power sampling & $\mathbf{0.133}$ & $-0.042$ & $0.133$ & $-0.042$ \\
& & Power-distilled / Standard & $\mathbf{0.133}$ & $-0.091$ & $\mathbf{0.200}$ & $-0.052$ \\
& & Power-distilled / Temperature & $0.117$ & $\mathbf{-0.038}$ & $0.100$ & $\mathbf{-0.026}$ \\
\midrule
\multirow{10}{*}{Phi-3.5-mini-instruct}
& \multirow{5}{*}{AMC23}
& Base / Standard & $0.169$ & $-0.232$ & $0.150$ & $-0.172$ \\
& & Base / Temperature & $0.156$ & $-0.112$ & $0.200$ & $-0.084$ \\
& & Power sampling & $\mathbf{0.225}$ & $-0.075$ & $0.225$ & $-0.075$ \\
& & Power-distilled / Standard & $0.194$ & $-0.049$ & $0.175$ & $-0.034$ \\
& & Power-distilled / Temperature & $0.181$ & $\mathbf{-0.030}$ & $\mathbf{0.250}$ & $\mathbf{-0.023}$ \\
\cmidrule(lr){2-7}
& \multirow{5}{*}{AIME24}
& Base / Standard & $0.017$ & $-0.269$ & $\mathbf{0.033}$ & $-0.187$ \\
& & Base / Temperature & $\mathbf{0.033}$ & $-0.126$ & $\mathbf{0.033}$ & $-0.093$ \\
& & Power sampling & $\mathbf{0.033}$ & $-0.080$ & $\mathbf{0.033}$ & $-0.080$ \\
& & Power-distilled / Standard & $\mathbf{0.033}$ & $-0.054$ & $\mathbf{0.033}$ & $-0.035$ \\
& & Power-distilled / Temperature & $0.017$ & $\mathbf{-0.033}$ & $\mathbf{0.033}$ & $\mathbf{-0.024}$ \\
\bottomrule
\end{tabular}}
\end{table}

\clearpage

\subsection[Power infinity]{Power$(\infty)$}
\label{sec_app:power-infty}

We also evaluated Power$(\infty)$ using Qwen2.5-Math-7B on MATH500.
This variant runs the MH power-sampling loop and accepts a proposal \(y'\) if and only if \(\pi(y'\mid x)>\pi(y\mid x)\) (\Cref{alg:samp}), corresponding to the limit \(\alpha\to\infty\).

\begin{table}[htbp]
\centering
\caption{Power$(\infty)$ results for Qwen2.5-Math-7B on MATH500. Left: all completions, means with $\pm$ std over seeds. Right: self-reward Best-of-\(N\) over samples generated with different seeds.}
\label{tab:power-infty}
\resizebox{\linewidth}{!}{
\begin{tabular}{lcccc}
\toprule
& \multicolumn{2}{c}{All completions} & \multicolumn{2}{c}{Self-reward Best-of-\(N\)} \\
\cmidrule(lr){2-3} \cmidrule(lr){4-5}
Sampling & $r^{\star} (\uparrow)$ & $r_{\mathrm{self}}$ & $r^{\star} (\uparrow)$ & $r_{\mathrm{self}}$ \\
\midrule
Power$(\infty)$ & $\mathbf{0.728 \pm 0.012}$ & $-0.075 \pm 0.001$ & $0.736$ & $-0.061$ \\
\bottomrule
\end{tabular}}
\end{table}

\clearpage

\subsection{Self-reward GRPO and On-policy distillation}
We compared power self-distillation with self-reward GRPO on MATH500 using Qwen2.5-Math-7B.
The GRPO reward is the same self-reward signal used in the theory, and we report several KL coefficients \(\beta\).
\Cref{tab:math500-main-and-grpo} shows that both GRPO and on-policy distillation (OPD) improve \(r_{\mathrm{self}}\) over the base model under standard decoding.
However, their downstream gains are smaller in this setting than those of power sampling and power self-distillation reported in \Cref{tab:logp-seed-aggregates}.

\begin{table}[htbp]
\centering
\caption{MATH500 comparison for Qwen2.5-Math-7B. We report downstream true reward~$r^{\star}$ (accuracy) and self-reward~$r_{\mathrm{self}}$. Left: all completions, mean with $\pm$ std over four seeds except OPD, which uses one checkpoint. Right: self-reward Best-of-\(N\) among seeds.}
\label{tab:math500-main-and-grpo}
\resizebox{\linewidth}{!}{
\begin{tabular}{llcccc}
\toprule
& & \multicolumn{2}{c}{All completions} & \multicolumn{2}{c}{Self-reward Best-of-\(N\)} \\
\cmidrule(lr){3-4} \cmidrule(lr){5-6}
Model / Checkpoint & Sampling & $r^{\star} (\uparrow)$ & $r_{\mathrm{self}}$ & $r^{\star} (\uparrow)$ & $r_{\mathrm{self}}$ \\
\midrule
\multirow{2}{*}{GRPO ($\beta{=}0.04$)} & Standard & $0.522 \pm 0.004$ & $-0.197 \pm 0.029$ & $0.698$ & $-0.099$ \\
 & Temperature & $0.690 \pm 0.005$ & $-0.051 \pm 0.000$ & $0.750$ & $-0.036$ \\
\midrule
\multirow{2}{*}{GRPO ($\beta{=}0.10$)} & Standard & $0.495 \pm 0.006$ & $-0.217 \pm 0.031$ & $0.712$ & $-0.112$ \\
 & Temperature & $0.681 \pm 0.005$ & $-0.082 \pm 0.002$ & $0.728$ & $-0.045$ \\
\midrule
\multirow{2}{*}{GRPO ($\beta{=}0.33$)} & Standard & $0.499 \pm 0.004$ & $-0.224 \pm 0.040$ & $0.666$ & $-0.108$ \\
 & Temperature & $0.699 \pm 0.006$ & $-0.071 \pm 0.001$ & $0.740$ & $-0.043$ \\
\midrule
 OPD & Standard & $0.496 \pm 0.014$ & $-0.272 \pm 0.010$ & $0.610$ & $-0.124$ \\
\bottomrule
\end{tabular}}
\end{table}

\paragraph{On-policy distillation efficiency.}
We also compared offline power self-distillation with an on-policy distillation variant, which still requires generation inside the training loop.
\Cref{tab:on-policy-distill-comparison} shows that, once the offline teacher dataset is fixed, power self-distillation performs only parameter updates and completes one epoch in \(183.8 \pm 0.3\) seconds.
OPD is about \(31\times\) slower per update in this setup, which reinforces the motivation for amortizing power-sampling cost into a reusable offline dataset.

\begin{table}[htbp]
\centering
\caption{Training-time comparison for one epoch on 500 MATH prompts with Qwen2.5-Math-7B. We report mean \(\pm\) standard deviation over two runs.}
\label{tab:on-policy-distill-comparison}
\resizebox{\linewidth}{!}{
\begin{tabular}{lcccc}
\toprule
Training scheme & Generation time & Parameter-update time & Total time & Time / update \\
\midrule
Power self-distillation & -- & $183.8 \pm 0.3$s & $183.8 \pm 0.3$s & $0.369 \pm 0.001$s \\
On-policy distillation & $5587.0 \pm 1539.1$s & $120.2 \pm 1.1$s & $5707.1 \pm 1540.2$s & $11.41 \pm 3.08$s \\
\bottomrule
\end{tabular}}
\end{table}

\clearpage

\subsection{Qualitative results}
\label{sec_app:qualitative}

This section presents full completions for one MATH-style geometry problem summarized in \Cref{tab:qualitative_math} with the gold answer $x+y=17$.
The prompt is:
\begin{quote}\small\itshape
The coordinates of a parallelogram are $(5,3)$, $(6,8)$, $(7,4)$, and $(x,y)$ with $x>7$.
What is the value of $x+y$?
\end{quote}

\begin{figure}[htbp]
  \centering
  \begingroup
  \fvset{
    fontsize=\relsize{-1},
    frame=single,
    framerule=0.75pt,
    framesep=3mm,
    rulecolor=\color{qualoutframe},
    breaklines=true,
    breakanywhere=false,
    breakautoindent=false,
    breakindent=0pt,
    breaksymbol={},
  }
  \captionsetup[subfigure]{position=top, labelfont=bf, font=footnotesize}
  \begin{subfigure}{\linewidth}
\begin{Verbatim}
The coordinates of a parallelogram are $(5,3)$, $(6,8)$, $(7,4)$, and $(x,y)$ with $x>7$.
What is the value of $x+y$?
\end{Verbatim}
    \caption{Question.}\label{subfig:qual_ex1_question}
  \end{subfigure}
  \begin{subfigure}{\linewidth}
\begin{Verbatim}
To find the missing coordinate of the parallelogram, we can use the property that opposite sides of a parallelogram are equal and parallel. Let's denote the given points as A(5, 3), B(6, 8), and C(7, 4). The missing point is D(x, y).

First, we need to find the vector AB and apply it to point C to find the coordinates of D. The vector AB is obtained by subtracting the coordinates of A from B: AB = (6-5, 8-3) = (1, 5).

To find the coordinates of D, we add the vector AB to the coordinates of C: D = C + AB = (7 + 1, 4 + 5) = (8, 9).

Therefore, the value of $x + y$ is $8 + 9 = 17$.

\end{Verbatim}
    \caption{Base sampling.}\label{subfig:qual_ex1_plain}
  \end{subfigure}
  \par\vspace{0.75em}
  \begin{subfigure}{\linewidth}
\begin{Verbatim}
To find the value of $x + y$, we first need to determine the coordinates of the fourth vertex of the parallelogram. Since the opposite sides of a parallelogram are equal and parallel, we can find the fourth vertex by using the properties of parallelograms.

Let's denote the fourth vertex as $(x, y)$. Since the opposite sides are equal, we can write the following equations based on the given coordinates:

1. $x - 5 = 6 - x$
2. $y - 3 = 8 - y$

Solving these equations, we get:

1. $2x = 11 \Rightarrow x = \frac{11}{2}$
2. $2y = 11 \Rightarrow y = \frac{11}{2}$

However, since $x > 7$, we need to adjust our solution. The correct coordinates for the fourth vertex should be $(8, 6)$ to satisfy the condition $x > 7$.

Therefore, the value of $x + y$ is $8 + 6 = 14$.
\end{Verbatim}
    \caption{Temperature sampling.}\label{subfig:qual_ex1_temp}
  \end{subfigure}
  \endgroup
  \caption{Full generations for an example in MATH500 (gold $x+y=17$) (Part 1/4).}
  \label{fig:qualitative_ex1_traces}
\end{figure}

\begin{figure}[p]
\ContinuedFloat
  \centering
  \begingroup
  \fvset{
    fontsize=\relsize{-2},
    frame=single,
    framerule=0.5pt,
    framesep=0.5mm,
    rulecolor=\color{qualoutframe},
    breaklines=true,
    breakanywhere=false,
    breakautoindent=false,
    breakindent=0pt,
    breaksymbol={},
  }
  \begin{subfigure}{\linewidth}
\begin{Verbatim}
Can you solve the following math problem? The coordinates of a parallelogram are (5, 3), (6, 8), (7, 4) and $(x, y)$ and $x > 7$. What is the value of $x + y$? Please reason step by step, and put your final answer within \boxed{{}}. To solve this problem, we need to use the properties of a parallelogram. In a parallelogram, the diagonals bisect each other. This means that the midpoint of one diagonal is the same as the midpoint of the other diagonal.

Given the coordinates of three vertices of the parallelogram as (5, 3), (6, 8), and (7, 4), let's denote the fourth vertex as $(x, y)$. The midpoint of the diagonal connecting (5, 3) and $(x, y)$ should be the same as the midpoint of the diagonal connecting (6, 8) and (7, 4).

The midpoint of a line segment with endpoints $(x_1, y_1)$ and $(x_2, y_2)$ is given by:
\[
\left( \frac{x_1 + x_2}{2}, \frac{y_1 + y_2}{2} \right)
\]

So, the midpoint of the diagonal connecting (5, 3) and $(x, y)$ is:
\[
\left( \frac{5 + x}{2}, \frac{3 + y}{2} \right)
\]

And the midpoint of the diagonal connecting (6, 8) and (7, 4) is:
\[
\left( \frac{6 + 7}{2}, \frac{8 + 4}{2} \right) = \left( \frac{13}{2}, 6 \right)
\]

Since these midpoints are the same, we can set up the following equations:
\[
\frac{5 + x}{2} = \frac{13}{2}
\]
\[
\frac{3 + y}{2} = 6
\]

Solving these equations will give us the values of $x$ and $y$.

Let's solve these equations step by step using Python.

```python
from sympy import symbols, Eq, solve

# Define the variables
x, y = symbols('x y')

# Define the equations based on the midpoints
eq1 = Eq((5 + x) / 2, 13 / 2)
eq2 = Eq((3 + y) / 2, 6)

# Solve the equations
solution = solve((eq1, eq2), (x, y))
print(solution)
```

```output
{x: 8.00000000000000, y: 9.00000000000000}
```

The solution to the equations is $x = 8$ and $y = 9$. Therefore, the coordinates of the fourth vertex are $(8, 9)$.

Now, we need to find the value of $x + y$. Since $x = 8$ and $y = 9$, we have:
\[
x + y = 8 + 9 = 17
\]

So, the value of $x + y$ is $\boxed{17}$.
\end{Verbatim}
    \caption{Power sampling.}\label{subfig:qual_ex1_power}
  \end{subfigure}
  \endgroup
  \caption{Full generations for an example in MATH500 (gold $x+y=17$) (Part 2/4).}
  \label{fig:qualitative_ex1_traces2}
\end{figure}

\begin{figure}[p]
\ContinuedFloat
  \centering
  \begingroup
  \fvset{
    fontsize=\relsize{-2},
    frame=single,
    framerule=0.5pt,
    framesep=0.5mm,
    rulecolor=\color{qualoutframe},
    breaklines=true,
    breakanywhere=false,
    breakautoindent=false,
    breakindent=0pt,
    breaksymbol={},
  }
  \begin{subfigure}{\linewidth}
\begin{Verbatim}
To solve for the coordinates \((x, y)\) of the parallelogram and find the value of \(x + y\), we need to use the properties of parallelograms. Specifically, the diagonals of a parallelogram bisect each other. This means that the midpoint of one diagonal is equal to the midpoint of the other diagonal.

Given the coordinates of three vertices of the parallelogram: \((5, 3)\), \((6, 8)\), and \((7, 4)\), we need to find the coordinates of the fourth vertex \((x, y)\).

First, let's find the midpoint of the diagonal formed by the vertices \((5, 3)\) and \((7, 4)\):
\[
\text{Midpoint} = \left( \frac{5 + 7}{2}, \frac{3 + 4}{2} \right) = \left( \frac{12}{2}, \frac{7}{2} \right) = (6, 3.5)
\]

Next, we find the midpoint of the diagonal formed by the vertices \((6, 8)\) and \((x, y)\):
\[
\text{Midpoint} = \left( \frac{6 + x}{2}, \frac{8 + y}{2} \right)
\]

Since the diagonals bisect each other, these midpoints must be equal:
\[
\left( \frac{6 + x}{2}, \frac{8 + y}{2} \right) = (6, 3.5)
\]

We now set up the equations by equating the corresponding coordinates:
\[
\frac{6 + x}{2} = 6 \quad \text{and} \quad \frac{8 + y}{2} = 3.5
\]

Solving the first equation for \(x\):
\[
\frac{6 + x}{2} = 6
\]
\[
6 + x = 12
\]
\[
x = 6
\]

Solving the second equation for \(y\):
\[
\frac{8 + y}{2} = 3.5
\]
\[
8 + y = 7
\]
\[
y = -1
\]

Thus, the coordinates of the fourth vertex are \((6, -1)\). However, we are given that \(x > 7\), which means there might be an error in our initial assumption or calculation. Let's re-evaluate the problem.

Given \(x > 7\), we need to consider the other possible diagonal pair. Let's use the vertices \((5, 3)\) and \((x, y)\) for one diagonal and \((6, 8)\) and \((7, 4)\) for the other diagonal.

First, find the midpoint of the diagonal formed by \((6, 8)\) and \((7, 4)\):
\[
\text{Midpoint} = \left( \frac{6 + 7}{2}, \frac{8 + 4}{2} \right) = \left( \frac{13}{2}, \frac{12}{2} \right) = (6.5, 6)
\]

\end{Verbatim}
    \caption{Power-distilled model with base sampling (Part 1/2).}\label{subfig:qual_ex1_distill}
  \end{subfigure}
  \endgroup
  \caption{Full generations for an example in MATH500 (gold $x+y=17$) (Part 3/4).}
  \label{fig:qualitative_ex1_traces3}
\end{figure}

\begin{figure}[p]
\ContinuedFloat
  \centering
  \begingroup
  \fvset{
    fontsize=\relsize{-1},
    frame=single,
    framerule=0.5pt,
    framesep=0.5mm,
    rulecolor=\color{qualoutframe},
    breaklines=true,
    breakanywhere=false,
    breakautoindent=false,
    breakindent=0pt,
    breaksymbol={},
  }
  \begin{subfigure}{\linewidth}
\begin{Verbatim}
Next, find the midpoint of the diagonal formed by \((5, 3)\) and \((x, y)\):
\[
\text{Midpoint} = \left( \frac{5 + x}{2}, \frac{3 + y}{2} \right)
\]

Since the diagonals bisect each other, these midpoints must be equal:
\[
\left( \frac{5 + x}{2}, \frac{3 + y}{2} \right) = (6.5, 6)
\]

We now set up the equations by equating the corresponding coordinates:
\[
\frac{5 + x}{2} = 6.5 \quad \text{and} \quad \frac{3 + y}{2} = 6
\]

Solving the first equation for \(x\):
\[
\frac{5 + x}{2} = 6.5
\]
\[
5 + x = 13
\]
\[
x = 8
\]

Solving the second equation for \(y\):
\[
\frac{3 + y}{2} = 6
\]
\[
3 + y = 12
\]
\[
y = 9
\]

Thus, the coordinates of the fourth vertex are \((8, 9)\). Since \(x > 7\), this solution is valid.

Finally, we find \(x + y\):
\[
x + y = 8 + 9 = 17
\]

Therefore, the value of \(x + y\) is \(\boxed{17}\).
\end{Verbatim}
    \caption{Power-distilled model with base sampling (Part 2/2).}\label{subfig:qual_ex1_distill2}
  \end{subfigure}
  \endgroup
  \caption{Full generations for an example in MATH500 (gold $x+y=17$) (Part 4/4).}
  \label{fig:qualitative_ex1_traces4}
\end{figure}
\clearpage
\subsection{Synthetic validation of suffix-Rényi odds corrections}
\label{sec_app:synthetic-renyi-validation}

To validate \Cref{prop:pow-vs-temp-renyi} in a setting that reflects the Zipf-like word-frequency structure of natural language, we construct a finite synthetic autoregressive distribution whose language-model next-token probabilities follow a Zipf-like law over many candidates.
Unlike the extreme pivotal-token construction of \citet{karan2026reasoning}, every next-token candidate is followed by a full-support suffix distribution.
The construction is summarized in \Cref{fig:renyi-power-temperature-setup}.
The base next-token distribution has \(V=64\) tokens with Zipf-like probabilities
\[
p_i \propto (i+1)^{-1.05}, \qquad i=0,\dots,V-1.
\]
For every token \(i\), the conditional suffix distribution \(q_i\) has the same support size \(M=256\), no zero-probability suffixes, and a non-uniform power-law shape
\[
q_i(z) \propto z^{-s_i}, \qquad z=1,\dots,M.
\]
The suffix exponent \(s_i\in[0.45,1.65]\) varies deterministically and non-monotonically with the next-token rank, using a sinusoidal component plus a small trend.
Thus, all suffix distributions have identical support size and full support, but differ in sharpness.
This deliberately avoids the singular-versus-uniform example in \citet{karan2026reasoning}: the experiment isolates the more general quantity identified by \Cref{prop:pow-vs-temp-renyi}, namely the suffix Rényi entropy.
In \Cref{fig:renyi-power-temperature-setup}, the left panel shows the Zipf-like next-token distribution, the middle panel shows the token-dependent suffix exponent \(s_i\), and the right panel shows representative full-support suffix distributions.

For each \(\alpha\in\{1.1,1.5,2,3,4,8\}\), we compute both the token-wise temperature next-token distribution and the sequence-level power next-token conditional exactly under this synthetic distribution.
The temperature next-token distribution is
\[
\pi_{\mathrm{temp},\alpha}(i)
=
\frac{p_i^\alpha}{\sum_j p_j^\alpha},
\]
whereas the next-token conditional induced by the sequence-level power distribution is
\[
\pi_{\mathrm{pow},\alpha}(i)
=
\frac{p_i^\alpha \sum_z q_i(z)^\alpha}
{\sum_j p_j^\alpha \sum_z q_j(z)^\alpha}.
\]
\Cref{fig:renyi-power-temperature-identity} compares the two sides of \Cref{prop:pow-vs-temp-renyi} for every unordered token pair and every tested \(\alpha\).
The left panel plots the Rényi-predicted log odds correction against the directly computed power-versus-temperature log odds correction, while the right panel shows the distribution of these corrections at the main experimental exponent \(\alpha=4\).

\Cref{fig:renyi-power-temperature-rank-reversals} illustrates the consequence of the correction at the level of next-token preferences: even when \(p_i>p_j\) and temperature favors token \(i\), sequence-level power can favor token \(j\) if \(q_j\) has sufficiently lower suffix Rényi entropy.

\begin{figure}[htbp]
\centering
\includegraphics[width=\linewidth]{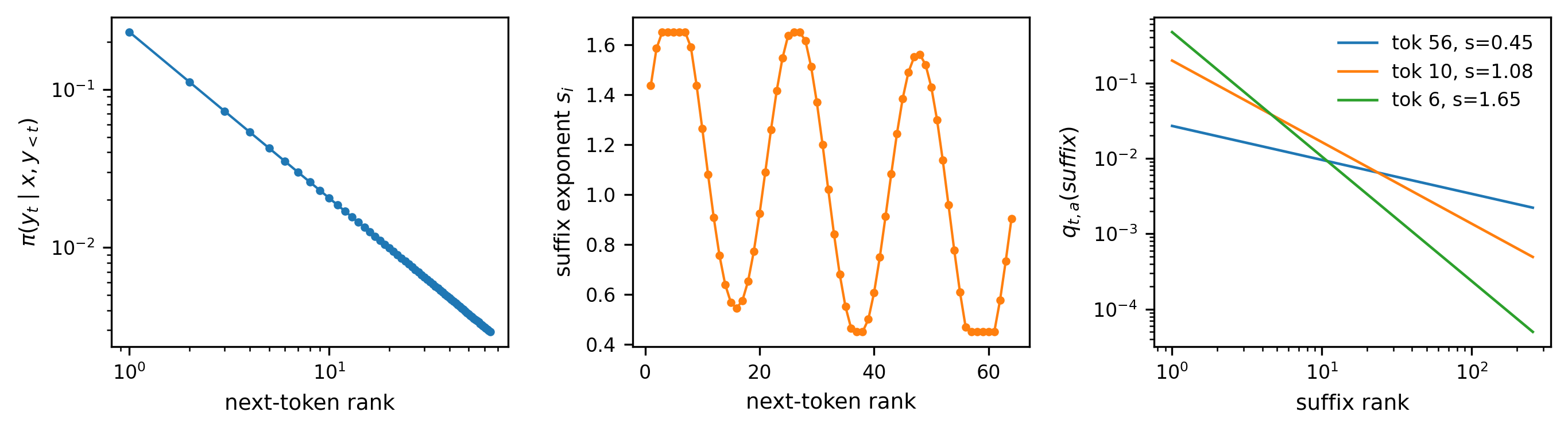}
\caption{Synthetic distribution setting.
Left: the base next-token distribution \(p_i\) is Zipf-like over \(64\) tokens.
Middle: suffix sharpness varies non-monotonically with the next-token rank through the power-law exponent \(s_i\).
Right: representative suffix distributions \(q_i\) are full-support, non-uniform power laws over the same support size \(M=256\).
This setup differs from the singular-versus-uniform pivotal-token construction of \citet{karan2026reasoning}.}
\label{fig:renyi-power-temperature-setup}
\end{figure}

\begin{figure}[htbp]
\centering
\includegraphics[width=0.85\linewidth]{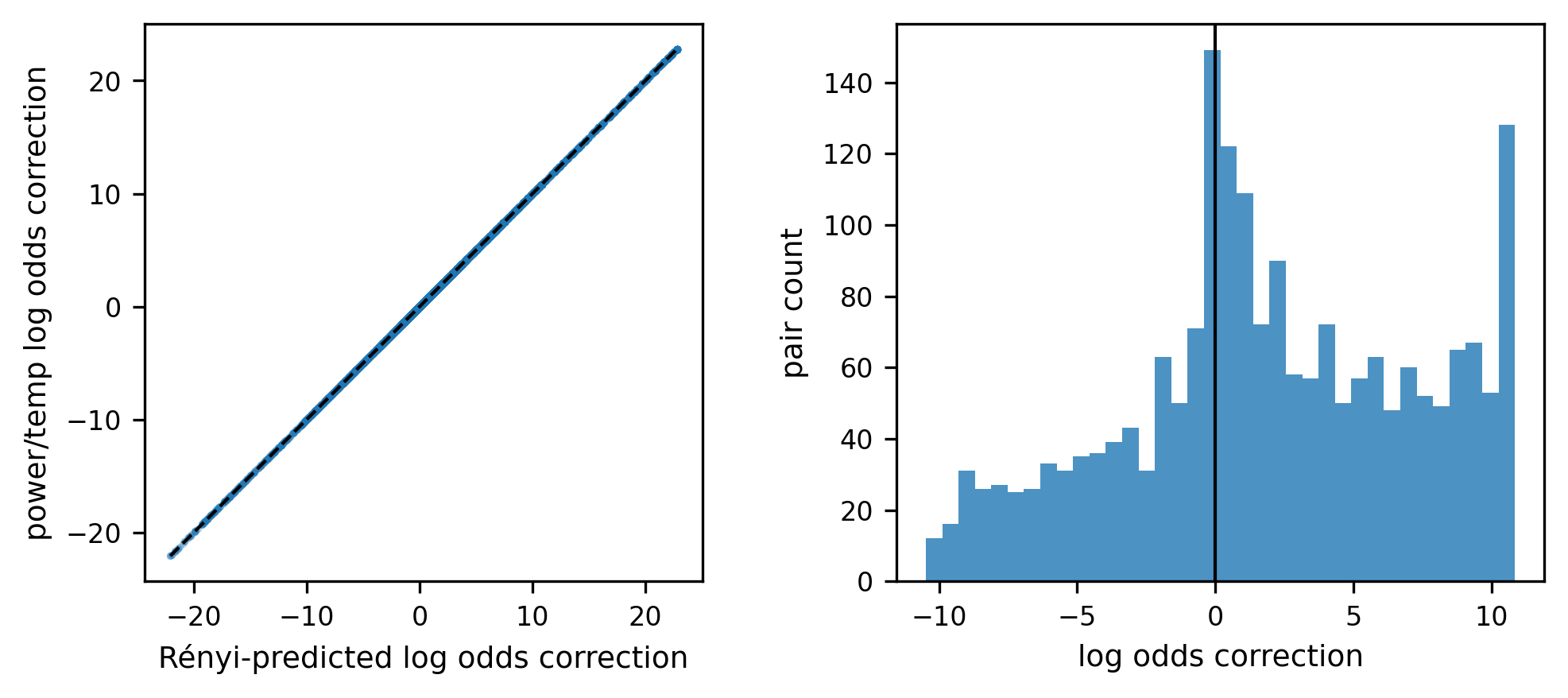}
\caption{Suffix Rényi entropy coincides with the power-vs-temperature odds gap.
For every unordered token pair \((i,j)\) and every tested \(\alpha\), we compare the Rényi-predicted log correction,
\((1-\alpha)(H_\alpha(q_i)-H_\alpha(q_j))\), with the closed-form log odds correction,
\(\log((\pi_{\mathrm{pow},\alpha}(i)/\pi_{\mathrm{pow},\alpha}(j))/(\pi_{\mathrm{temp},\alpha}(i)/\pi_{\mathrm{temp},\alpha}(j)))\).
The diagonal agreement shows that the suffix-Rényi formula explains the full pairwise gap in a many-token full-support setting.
The right panel shows the distribution of closed-form corrections at \(\alpha=4\).}
\label{fig:renyi-power-temperature-identity}
\end{figure}

\begin{figure}[htbp]
\centering
\includegraphics[width=\linewidth]{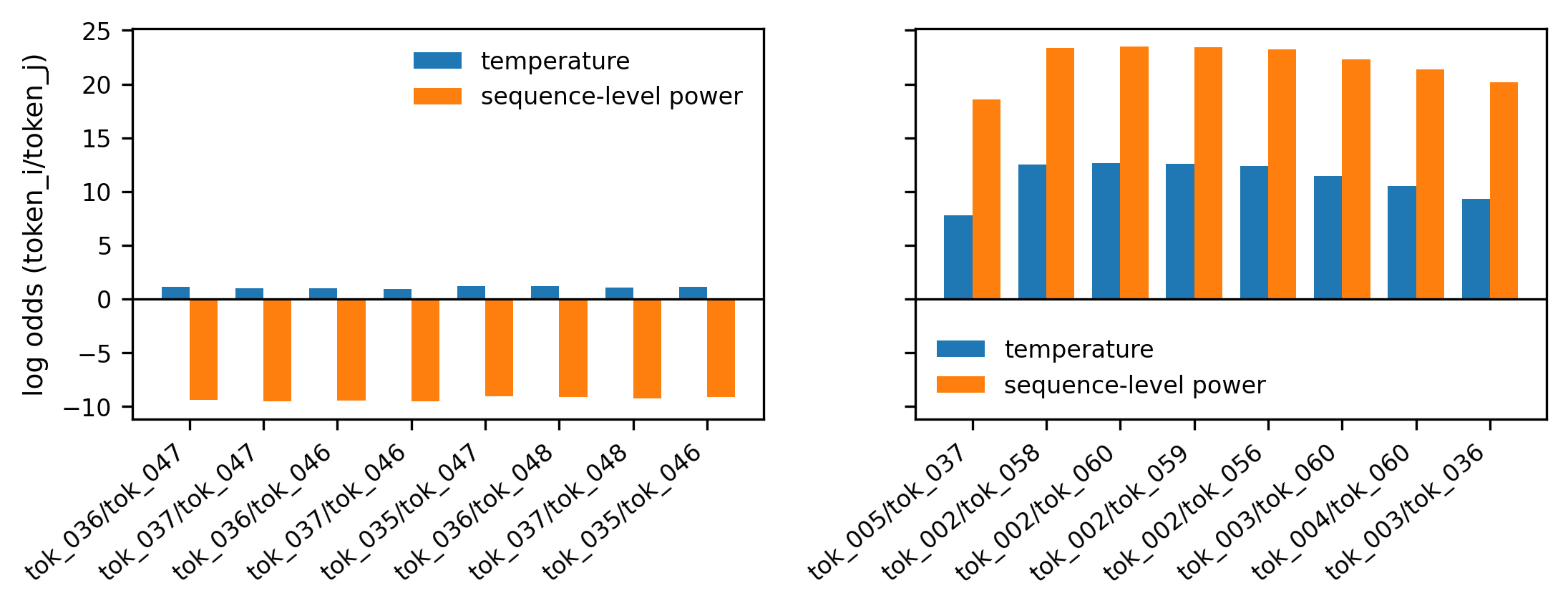}
\caption{Suffix entropy can reverse next-token preferences under sequence-level power.
Pairs are ordered so that \(i<j\), hence the Zipf base next-token distribution gives \(p_i>p_j\), and token-wise temperature has positive log odds for \(i\) over \(j\).
Left: examples where the sequence-level power next-token conditional reverses this ordering because token \(j\) has a sharper, lower-entropy suffix distribution.
Right: control examples where the Rényi entropy is not large enough to reverse the preferred token.
All bars are closed-form log odds, not sampled frequencies.}
\label{fig:renyi-power-temperature-rank-reversals}
\end{figure}

\clearpage
\subsection{Synthetic validation of optimal one-step proposals for sequential power sampling}
\label{sec_app:synthetic-sis-validation}

We reuse the synthetic distribution of \Cref{sec_app:synthetic-renyi-validation} to validate \Cref{prop:opt-proposal-power}.
For a fixed prompt and an empty prefix, the unique variance-minimizing one-step proposal in \Cref{eq:opt-q-power} reduces to
\[
q^\star(i)
\;\propto\;
p_i^{\alpha}\,\sum_{z=1}^{M} q_i(z)^{\alpha},
\qquad i=0,\dots,V-1,
\]
which equals the next-token conditional of the sequence-level power distribution \(\pi_{\mathrm{pow},\alpha}\) and depends on the suffix power masses \(\sum_z q_i(z)^{\alpha}\) of every candidate token.
We compare \(q^\star\) with three one-step proposals that do not use those suffix totals: the base proposal \(q^{\mathrm{base}}(i)=p_i\), the token-wise temperature proposal \(q^{\mathrm{temp}}(i)\propto p_i^{\alpha}\), and a uniform reference \(q^{\mathrm{unif}}(i)=1/V\).

For each proposal \(q\), the first-step incremental importance weight in \Cref{eq:inc-weight-def} simplifies to
\[
W_1(i)
\;=\;
\frac{q^\star(i)}{q(i)},
\]
and we show its exact mean, the coefficient of variation
\(\mathrm{CV}^2(W_1)=\mathrm{Var}[W_1]/\mathbb{E}[W_1]^2\),
and the effective sample size fraction
\(\mathrm{ESS}/N=1/(1+\mathrm{CV}^2(W_1))\).
By \Cref{prop:opt-proposal-power}, only \(q^\star\) achieves \(\mathrm{Var}[W_1]=0\) and hence \(\mathrm{ESS}/N=1\); the closed-form values for the other proposals are computed exactly from the synthetic distribution.

\Cref{fig:renyi-power-sis-proposal-comparison} compares the four proposals at \(\alpha=4\).
The left panel shows the proposal probabilities; the oracle proposal equals the target next-token conditional \(\pi_{\mathrm{pow},\alpha}\) by construction, and the temperature, base, and uniform proposals deviate from it, especially on next-token ranks where the suffix exponent \(s_i\) is small and \(\sum_z q_i(z)^{\alpha}\) is large.
The right panel plots \(\log W_1(i)\): only the oracle proposal yields a constant log weight, while the other proposals produce token-dependent log weights.

\Cref{fig:renyi-power-sis-exact-weight-variance} reports the exact \(\mathrm{ESS}/N\) and \(\mathrm{CV}^2(W_1)\) as a function of \(\alpha\).
The oracle proposal attains \(\mathrm{ESS}/N=1\) for every \(\alpha\), whereas the gap between the temperature proposal and the oracle widens as \(\alpha\) grows, because larger \(\alpha\) amplifies the suffix power masses that the local temperature transform ignores.

\Cref{fig:renyi-power-sis-monte-carlo-ess} checks the same conclusion with Monte Carlo: for each proposal we draw \(N\) tokens, compute the self-normalized \(\mathrm{ESS}\), and average across replicates.
The sampled \(\mathrm{ESS}/N\) concentrates around the exact values from \Cref{fig:renyi-power-sis-exact-weight-variance} as \(N\) grows, and the ordering of the proposals is preserved at every particle budget.

\begin{figure}[htbp]
\centering
\includegraphics[width=\linewidth]{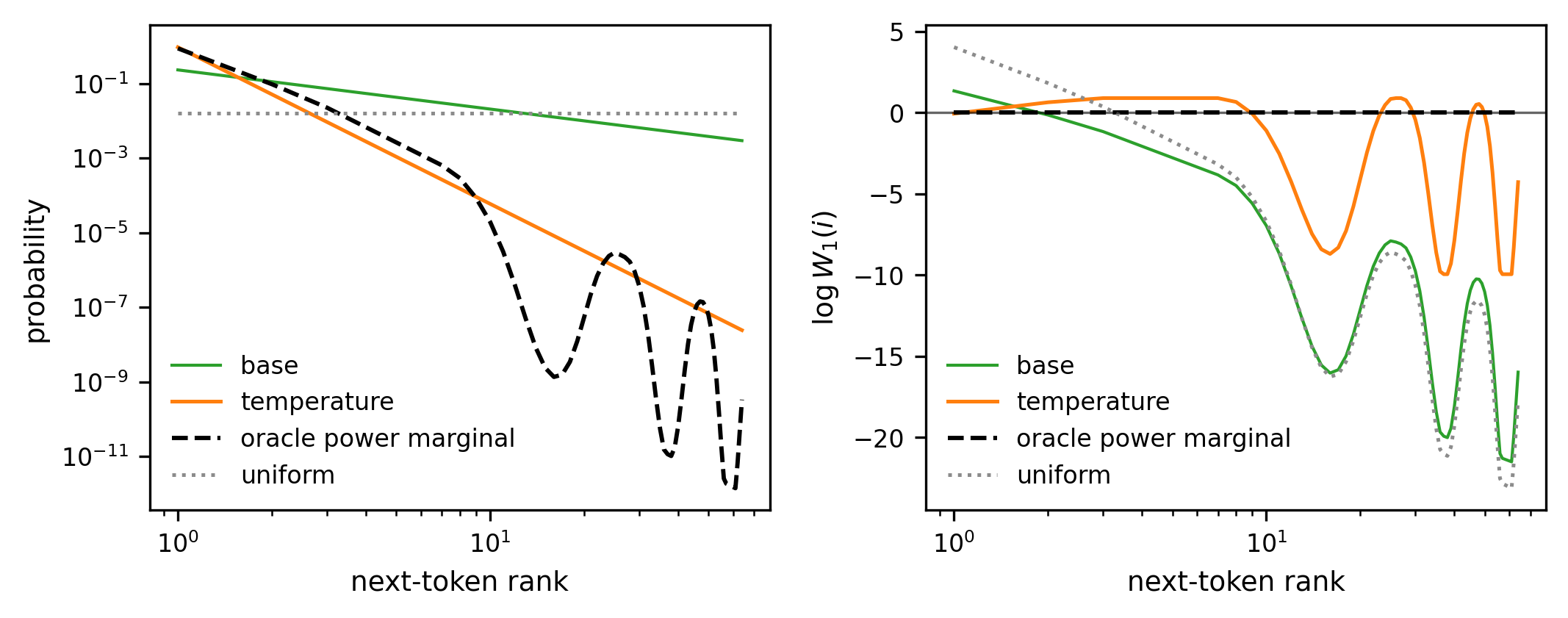}
\caption{One-step SIS proposals for the sequence-level power target.
Left: proposal probabilities at \(\alpha=4\). The oracle proposal \(q^\star\) of \Cref{eq:opt-q-power} equals the target next-token conditional \(\pi_{\mathrm{pow},\alpha}\) by construction, while the temperature, base, and uniform proposals deviate from it.
Right: log incremental weight \(\log W_1(i)\) for each proposal; the oracle yields a constant log weight, while base, temperature, and uniform proposals do not.}
\label{fig:renyi-power-sis-proposal-comparison}
\end{figure}

\begin{figure}[htbp]
\centering
\includegraphics[width=\linewidth]{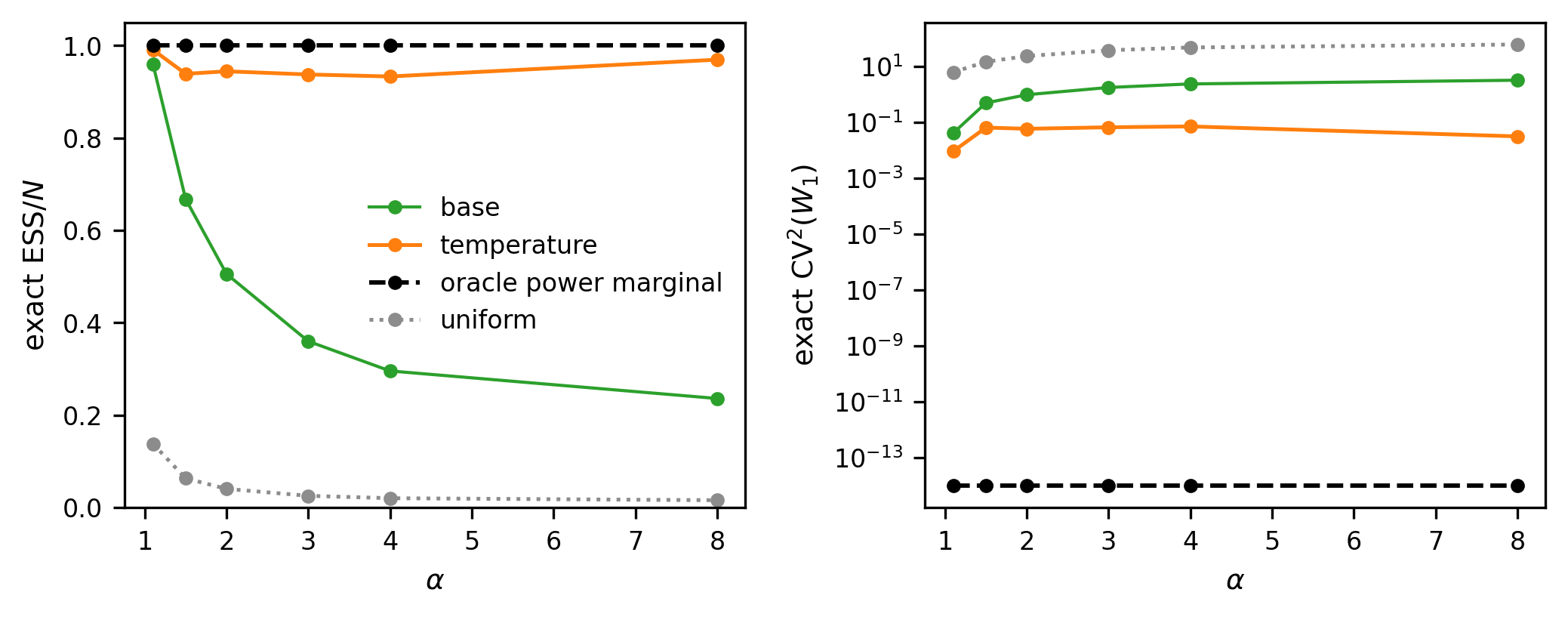}
\caption{Exact first-step weight variance and effective sample size.
Left: exact \(\mathrm{ESS}/N=1/(1+\mathrm{CV}^2(W_1))\) for each one-step proposal as a function of \(\alpha\).
Right: exact \(\mathrm{CV}^2(W_1)\) on a log scale.
The oracle proposal \(q^\star\) achieves \(\mathrm{Var}[W_1]=0\) at every \(\alpha\), confirming \Cref{prop:opt-proposal-power}, while the temperature proposal degrades as \(\alpha\) grows.}
\label{fig:renyi-power-sis-exact-weight-variance}
\end{figure}

\begin{figure}[htbp]
\centering
\includegraphics[width=0.6\linewidth]{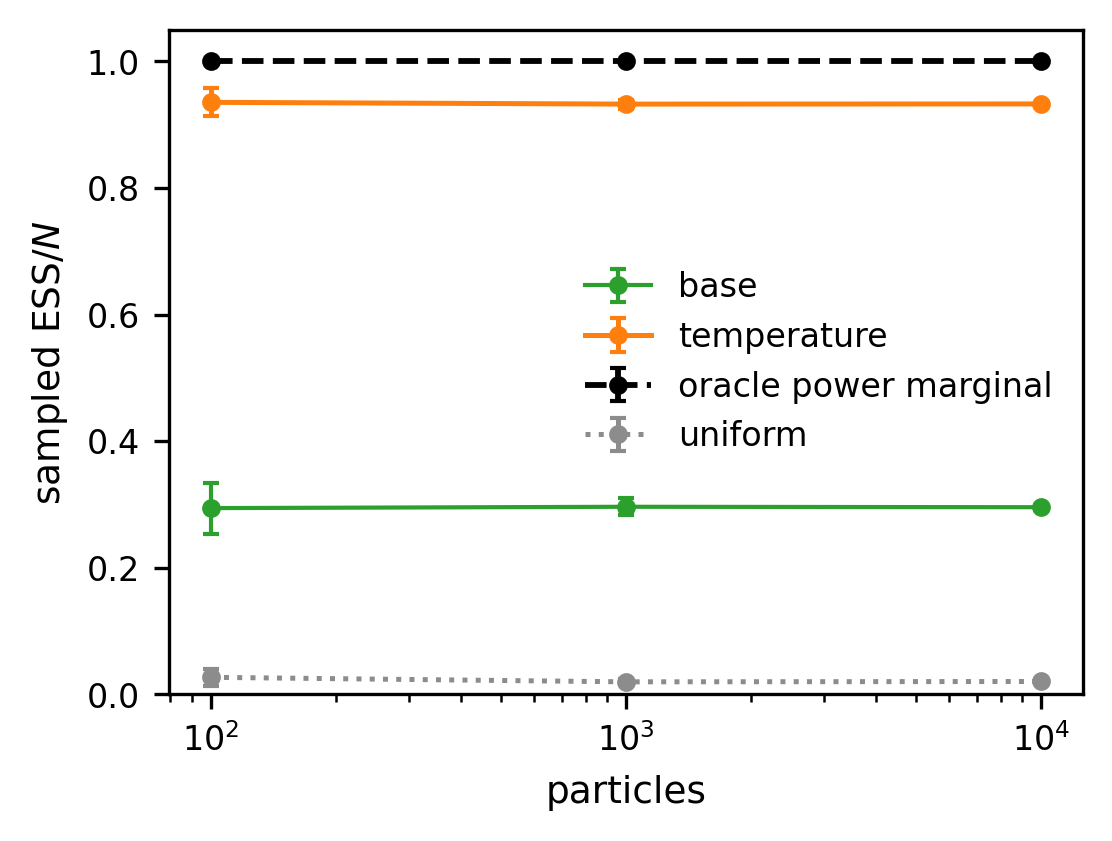}
\caption{Monte Carlo effective sample size at \(\alpha=4\).
For each proposal we draw \(N\) tokens, compute the self-normalized \(\mathrm{ESS}\), and average across replicates; error bars show one standard deviation.}
\label{fig:renyi-power-sis-monte-carlo-ess}
\end{figure}

\clearpage
\section*{NeurIPS Paper Checklist}

\begin{enumerate}

\item {\bf Claims}
    \item[] Question: Do the main claims made in the abstract and introduction accurately reflect the paper's contributions and scope?
    \item[] Answer: \answerYes{}.
    \item[] Justification: The abstract and introduction state the main theoretical and empirical claims: the structural difficulty of inexpensive power-sampling approximations, the derivation of power self-distillation from a KL-regularized self-reward RL objective, the sharpening and covariance characterizations, and the empirical comparison on reasoning tasks.

\item {\bf Limitations}
    \item[] Question: Does the paper discuss the limitations of the work performed by the authors?
    \item[] Answer: \answerYes{}.
    \item[] Justification: The conclusion includes a limitations paragraph discussing dependence on the base model's capabilities and the focus on autoregressive language models with finite completion horizons.

\item {\bf Theory assumptions and proofs}
    \item[] Question: For each theoretical result, does the paper provide the full set of assumptions and a complete (and correct) proof?
    \item[] Answer: \answerYes{}.
    \item[] Justification: The main theoretical statements specify their assumptions in the propositions and corollary, and the appendix provides the proofs for all results.

    \item {\bf Experimental result reproducibility}
    \item[] Question: Does the paper fully disclose all the information needed to reproduce the main experimental results of the paper to the extent that it affects the main claims and/or conclusions of the paper (regardless of whether the code and data are provided or not)?
    \item[] Answer: \answerYes{}.
    \item[] Justification: The paper describes the evaluated models, datasets, splits, power-sampling procedure, distillation setup, LoRA configuration, optimizer, and hyperparameters in the main text and appendix. These details are intended to make the reported comparisons reproducible.

\item {\bf Open access to data and code}
    \item[] Question: Does the paper provide open access to the data and code, with sufficient instructions to faithfully reproduce the main experimental results, as described in supplemental material?
    \item[] Answer: \answerNo{}.
    \item[] Justification: The experiments use publicly available datasets and models cited in the paper, but the current submission does not include public code or generated data to preserve anonymity. We plan to release the code and reproduction instructions upon acceptance.

\item {\bf Experimental setting/details}
    \item[] Question: Does the paper specify all the training and test details (e.g., data splits, hyperparameters, how they were chosen, type of optimizer) necessary to understand the results?
    \item[] Answer: \answerYes{}.
    \item[] Justification: The experimental setup specifies the models, datasets, evaluation subsets, distillation data sources, power-sampling parameters, baselines, LoRA configuration, optimizer, learning rates, sequence length, batch size, gradient accumulation, and training schedule.

\item {\bf Experiment statistical significance}
    \item[] Question: Does the paper report error bars suitably and correctly defined or other appropriate information about the statistical significance of the experiments?
    \item[] Answer: \answerYes{}.
    \item[] Justification: The main and appendix tables report means with standard deviations, and the covariance experiment reports standard-deviation error bars.

\item {\bf Experiments compute resources}
    \item[] Question: For each experiment, does the paper provide sufficient information on the computer resources (type of compute workers, memory, time of execution) needed to reproduce the experiments?
    \item[] Answer: \answerYes{}.
    \item[] Justification: \Cref{sec_app:experimental-setups} reports the hardware configuration used for the experiments (CPU type, host memory, GPU type, number of GPUs per node, and GPU memory) together with execution times.

\item {\bf Code of ethics}
    \item[] Question: Does the research conducted in the paper conform, in every respect, with the NeurIPS Code of Ethics \url{https://neurips.cc/public/EthicsGuidelines}?
    \item[] Answer: \answerYes{}.
    \item[] Justification: The work analyzes and evaluates post-training and inference methods for publicly available language models on standard academic benchmarks, and we are not aware of any deviation from the NeurIPS Code of Ethics.

\item {\bf Broader impacts}
    \item[] Question: Does the paper discuss both potential positive societal impacts and negative societal impacts of the work performed?
    \item[] Answer: \answerNA{}.
    \item[] Justification: This paper is foundational research on theoretical and empirical properties of power sampling and self-distillation, and does not introduce a deployed system, dataset release, or application-specific intervention with direct societal impact.

\item {\bf Safeguards}
    \item[] Question: Does the paper describe safeguards that have been put in place for responsible release of data or models that have a high risk for misuse (e.g., pre-trained language models, image generators, or scraped datasets)?
    \item[] Answer: \answerNA{}.
    \item[] Justification: The paper does not release a new pretrained language model, scraped dataset, or other high-risk asset. The experiments use existing public models and benchmarks.

\item {\bf Licenses for existing assets}
    \item[] Question: Are the creators or original owners of assets (e.g., code, data, models), used in the paper, properly credited and are the license and terms of use explicitly mentioned and properly respected?
    \item[] Answer: \answerYes{}.
    \item[] Justification: The paper uses existing public models and datasets, and properly credits them through citations to the original papers and resources in the experimental setup.

\item {\bf New assets}
    \item[] Question: Are new assets introduced in the paper well documented and is the documentation provided alongside the assets?
    \item[] Answer: \answerNA{}.
    \item[] Justification: The paper introduces a method and reports experiments, but the current submission does not release a new dataset, code package, or model checkpoint as a new asset.

\item {\bf Crowdsourcing and research with human subjects}
    \item[] Question: For crowdsourcing experiments and research with human subjects, does the paper include the full text of instructions given to participants and screenshots, if applicable, as well as details about compensation (if any)?
    \item[] Answer: \answerNA{}.
    \item[] Justification: The work does not involve crowdsourcing experiments or research with human subjects.

\item {\bf Institutional review board (IRB) approvals or equivalent for research with human subjects}
    \item[] Question: Does the paper describe potential risks incurred by study participants, whether such risks were disclosed to the subjects, and whether Institutional Review Board (IRB) approvals (or an equivalent approval/review based on the requirements of your country or institution) were obtained?
    \item[] Answer: \answerNA{}.
    \item[] Justification: The work does not involve crowdsourcing experiments or research with human subjects, so IRB approval is not applicable.

\item {\bf Declaration of LLM usage}
    \item[] Question: Does the paper describe the usage of LLMs if it is an important, original, or non-standard component of the core methods in this research? Note that if the LLM is used only for writing, editing, or formatting purposes and does \emph{not} impact the core methodology, scientific rigor, or originality of the research, declaration is not required.
    \item[] Answer: \answerNA{} 
    \item[] Justification: The core methodology of this research does not involve the use of LLMs in any important, original, or non-standard way.

\end{enumerate}

\end{document}